\documentclass[article,12pt]{article}
\RequirePackage[T1]{fontenc}
\usepackage{times}
\RequirePackage{amssymb,amsbsy,amsmath,amscd,amsthm,amsfonts,bbm, mathrsfs}
\RequirePackage[numbers]{natbib}
\RequirePackage[colorlinks,citecolor=blue,urlcolor=blue, hypertexnames=false]{hyperref}
\usepackage[margin=1.25in]{geometry}
\usepackage{graphicx}
\usepackage{enumerate,xspace}
\usepackage{amsfonts}
\usepackage{amssymb}
\usepackage{fancyhdr}
\usepackage{comment}
\usepackage{parskip}
\usepackage{float}
\usepackage{natbib}
\usepackage{longtable}
\usepackage{tablefootnote}
\usepackage{url}
\usepackage{standalone}
\usepackage{enumitem}
\restylefloat{table}
\usepackage{mathtools}

\newtheorem{theorem}{Theorem}

\newtheorem{lemma}{Lemma}
\newtheorem{proposition}{Proposition}
\newtheorem{remark}{Remark}

\newtheorem{assumption}{Assumption}

\usepackage{algpseudocode}

\usepackage{tikz}
\usetikzlibrary{matrix,arrows,calc,shapes,backgrounds}
\usetikzlibrary{shapes.callouts,decorations.text} 
\usetikzlibrary{shapes.misc}

\renewcommand{\tilde}{\widetilde}

\newcommand{\Expect}{\mathbb{E}}

\newcommand{\calD}{{\mathcal{D}}}

\newcommand{\calH}{{\mathcal{H}}}

\newcommand{\calN}{{\mathcal{N}}}
\newcommand{\calO}{{\mathcal{O}}}

\newcommand{\calS}{{\mathcal{S}}}

\usepackage{prettyref,xspace}
\newrefformat{eq}{(\ref{#1})}
\newrefformat{thm}{Theorem~\ref{#1}}
\newrefformat{fact}{Fact~\ref{#1}}
\newrefformat{sec}{Section~\ref{#1}}
\newrefformat{algo}{Algorithm~\ref{#1}}
\newrefformat{fig}{Fig.~\ref{#1}}
\newrefformat{tab}{Table~\ref{#1}}
\newrefformat{rmk}{Remark~\ref{#1}}
\newrefformat{def}{Definition~\ref{#1}}
\newrefformat{cor}{Corollary~\ref{#1}}
\newrefformat{lmm}{Lemma~\ref{#1}}
\newrefformat{prop}{Proposition~\ref{#1}}
\newrefformat{ex}{Example~\ref{#1}}
\newrefformat{ass}{Assumption~\ref{#1}}

\DeclareMathOperator*{\argmin}{arg\,min}
\DeclareMathOperator*{\argmax}{arg\,max}

\usepackage[font=small,labelfont=bf]{caption}
\usepackage[font=small,labelfont=bf]{subcaption}

\def\bt{\mathrm{t}}

\def\bX{\mathbf{X}}

\def\bzero{\mathbf{0}}

\def\bSigma{\boldsymbol{\Sigma}}

\newcommand{\Var}{\mathrm{Var}}
\newcommand{\Cov}{\mathrm{Cov}}

\newcommand{\BR}{\mathbb R}

\newcommand{\BE}{\mathbb E}
\newcommand{\BP}{\mathbb P}

\graphicspath{{./figs/}}

\usepackage[ruled]{algorithm2e}

\usepackage{authblk}

\begin{document}
\title{Nonparametric Variable Screening with Optimal Decision Stumps}
\author[ ]{Jason M. Klusowski\thanks{\href{mailto:jason.klusowski@princeton.edu}{jason.klusowski@princeton.edu} \\ Supported in part by NSF DMS-1915932 and NSF TRIPODS DATA-INSPIRE CCF-1934924.}}
\author[ ]{Peter M. Tian\thanks{\href{mailto:ptian@princeton.edu}{ptian@princeton.edu} \\ Supported in part by the Gordon Wu Fellowship of Princeton University.}}
\affil[ ]{Department of Operations Research and Financial Engineering, Princeton University}
\date{}
\maketitle

\begin{abstract}
Decision trees and their ensembles are endowed with a rich set of diagnostic tools for ranking and screening variables in a predictive model. Despite the widespread use of tree based variable importance measures, pinning down their theoretical properties has been challenging and therefore largely unexplored. To address this gap between theory and practice, we derive finite sample performance guarantees for variable selection in nonparametric models using a single-level CART decision tree (a decision stump). Under standard operating assumptions in variable screening literature, we find that the marginal signal strength of each variable and ambient dimensionality can be considerably weaker and higher, respectively, than state-of-the-art nonparametric variable selection methods. Furthermore, unlike previous marginal screening methods that attempt to directly estimate each marginal projection via a truncated basis expansion, the fitted model used here is a simple, parsimonious decision stump, thereby eliminating the need for tuning the number of basis terms. Thus, surprisingly, even though decision stumps are highly inaccurate for estimation purposes, they can still be used to perform consistent model selection.


\end{abstract}






\section{Introduction}

A common task in many applied disciplines involves determining which variables, among many, are most important in a predictive model.
In high-dimensional sparse models, many of these predictor variables may be irrelevant in how they affect the response variable. As a result, variable selection techniques are crucial for filtering out irrelevant variables in order to prevent overfitting, improve accuracy, and enhance the interpretability of the model. Indeed, algorithms that screen for relevant variables have been instrumental in the modern development of fields such as genomics, biomedical imaging, signal processing, image analysis, and finance, where high-dimensional, sparse data is frequently encountered \citep{fan2008sis}.

Over the years, numerous parametric and nonparametric methods for variable selection in high-dimensional models have been proposed and studied. 
For linear models, the LARS algorithm \citep{efron2004lars} (for Lasso \citep{tibshirani1996lasso}) 
and Sure Independence Screening (SIS) \citep{fan2008sis} serve as prototypical examples that have achieved immense success, both practically and theoretically. Other strategies for nonparametric additive 
models such as Nonparametric Independence Screening (NIS) \citep{fan2011nonparametric} and Sparse Additive Models (SPAM) \citep{ravikumar2009spam} have also enjoyed a similar history of success. 

\subsection{Tree based variable selection procedures}
Alternatively, 
because they are built from highly interpretable and simple objects,
decision tree models are another important tool in the data analyst's repertoire.
Indeed, after only a brief explanation, one is able to understand the tree construction and its output in terms of meaningful domain specific attributes of the variables. In addition to being interpretable, tree based model have good computational scalability as the number of data points grows, making them faster than many other methods when dealing with large datasets. In terms of flexibility, they can naturally handle a mixture of numeric variables, categorical variables, and missing values. Lastly, they require less preprocessing (because they are invariant to monotone transformations of the inputs), are quite robust to outliers, and are relatively unaffected by the inclusion of many irrelevant variables \citep{hastie2009elements, klusowski2020sparse}, the last point being of relevance to the variable selection problem.


Conventional tree structured models such as CART \citep{breiman1984}, random forests \citep{breiman2001}, ExtraTrees \citep{geurts2006extratrees}, and gradient tree boosting \citep{friedman2001greedy} are also equipped with heuristic variable importance measures that can be used to rank and identify relevant predictor variables for further investigation (such as plotting their partial dependence functions). 
In fact, tree based variable importance measures have been used 
to discover genes in
bioinformatics \citep{breiman2001statistical,lunetta2004screening,bureau2005snps,diaz2006gene,huynh-thu2010inferring}, 
identify loyal customers and clients \citep{buckinx2007predicting,buckinx2005customer,lariviere2005predicting}, detect network intrusion \citep{zhang2006hybrid,zhang2008hybrid}, and understand income redistribution preferences \citep{keely2008understanding}, to name a few applications. 

\subsection{Mean decrease in impurity} \label{mdi}

An attractive feature specific to CART methodology is that one can compute, essentially for free, measures of variable importance (or influence) using the optimal splitting variables and their corresponding impurities. The canonical CART-based variable importance measure is the Mean Decrease in Impurity (MDI) \citep[Section 8.1]{friedman2001greedy}, \citep[Section 5.3.4]{breiman1984}, \citep[Sections 10.13.1 \& 15.3.2]{hastie2009elements}, 
which calculates an importance score for a variable by summing the largest impurity reductions (weighted by the fraction of samples in the node) over all non-terminal nodes split with that variable, averaged over all trees in the ensemble. Another commonly used and sometimes more accurate measure is the Mean Decrease in Accuracy (MDA) or permutation importance measure, defined as the average difference in out-of-bag errors before and after randomly permuting the data values of a variable in out-of-bag samples over all trees in the ensemble \citep{breiman2001}. However, from a computational perspective, MDI is preferable to MDA since it can be computed as each tree is grown with no additional cost.
While we view the analysis of MDA as an important endeavor, it is outside the scope of the present paper and therefore not our current focus. 
In contrast to the aforementioned variable selection procedures like SIS, NIS, Lasso, and SPAM, except for a few papers, little is known about the finite sample performance of MDI. Theoretical results are mainly limited to the asymptotic, fixed dimensional data setting and to showing what would be expected from a reasonable measure of variable importance. For example, it is shown in \citep{louppe2013understanding} that the MDI importance of a categorical variable (in an asymptotic data and ensemble size setting) is zero precisely when the variable is irrelevant, and that the MDI importance for relevant variables remains unchanged if irrelevant variables are removed or added. 

Recent complementary work by \citep{scornet2020trees} established the large sample properties of MDI for additive models by showing that it converges to sensible quantities like the variance of the component functions, provided the decision tree is sufficiently deep. However, these results are not fine-grained enough to handle the case where either the dimensionality grows or the marginal signals decay with the sample size. Furthermore, \citep{scornet2020trees} crucially relies on the assumption that the CART decision tree is a \emph{consistent} predictor, which is currently only known for additive models \citep{scornet2015}. It is therefore unclear whether the techniques can be generalized to models beyond those considered therein.
On the other hand, our results suggest that tree based variable importance measures can still have good variable selection properties even though the underlying tree model may be a poor predictor of the data generating process—which can occur with CART and random forests.


Lastly, we mention that important steps have also been taken to characterize the finite sample properties of MDI; \citep{li2019debiased} show that MDI is less biased for irrelevant variables when each tree is shallow. This work therefore covers one facet of the variable selection problem, i.e., controlling the number of false positives, and will be employed in our proofs.

\subsection{New contributions}


The lack of theoretical development for tree based variable selection is likely because the training mechanism involves complex steps —which present new theoretical challenges— such as bagging, boosting, pruning, random selections of the predictor variables for candidate splits, recursive splitting, and line search to find the best split points \citep{kazemitabar2017variable}. The last consideration, importantly, means that the underlying tree construction (e.g., split points) depends on \emph{both} the input and output data,
which enables it to adapt to structural properties of the underlying statistical model (such as sparsity). This data adaptivity is a double-edged sword 
 from a theoretical standpoint, though, since unravelling the data dependence is a formidable
task.

Despite the aforementioned challenges, we advance the study of tree based variable selection by focusing on the following two fundamental questions:
\begin{itemize}
\item Does MDI enjoy finite sample guarantees for its variable ranking?
\item What are the benefits of MDI over other variable screening methods?
\end{itemize}

Specifically, we derive rigorous finite sample guarantees for what we call the SDI importance measure, a sobriquet for {\it Single-level Decrease in Impurity}, which is a special case of MDI for a single-level CART decision tree or ``decision stump'' \citep{iba1992stumps}. This is similar in spirit to the approach of DSTUMP \citep{kazemitabar2017variable}
but, importantly, SDI incorporates the line search step by finding the \emph{optimal} split point, instead of the empirical median, of every predictor variable.
As we shall see, ranking variables according to their SDI is equivalent to ranking the variables according to the marginal sample correlations between the response data and the optimal decision stump with respect to those variables. This equivalence also yields connections with other variable selection methods: for linear models with Gaussian variates, we show that SDI is asymptotically equivalent to SIS (up to logarithmic factors in the sample size), and so SDI inherits the so-called sure screening property \citep{fan2008sis} under suitable assumptions. 

Unlike SIS, however, SDI is accompanied by provable guarantees for nonparametric 
models. 
We show that under certain conditions, SDI achieves {\it model selection consistency}; that is, it correctly selects the relevant variables of the model with probability approaching one as the sample size increases. In fact, the minimum signal strength of each relevant variable and maximum dimensionality of the model
are shown to be less restrictive for SDI than NIS or SPAM. In the linear model case with Gaussian variates, SDI is shown to nearly match the optimal sample size threshold (achieved by Lasso) for exact support recovery. 
These favorable properties are particularly striking when one is reminded that the underlying model fit to the data is a simple, parsimonious decision stump—in particular, there is no need to specify a flexible function class (such as a polynomial spline family) and be concerned with calibrating the number of basis terms or bandwidth parameters. 

Finally, we empirically compare SDI to other contemporaneous variable selection algorithms, namely, SIS, NIS, Lasso, and SPAM. We find that SDI is competitive and performs favorably in cases where the model is less smooth.
Furthermore, we empirically verify the similarity between SDI and SIS for the linear model case, and confirm the model selection consistency properties of SDI for various types of nonparametric models.

Let us close this section by saying a few words about the primary goals of this paper. In practice, tree based variable importance measures such as MDI and MDA are most commonly used to \emph{rank variables}, in order to determine which ones are worthy of further examination. This is a less ambitious endeavor than that sought by the aforementioned variable selection methods, which aim for consistent \emph{model selection}, namely, determining exactly which variables are relevant and irrelevant. Though we do study the model selection problem, our priority is to demonstrate the power of tree based variable importance measures as interpretable, accurate, and efficient variable ranking tools.

\subsection{Organization} 

The paper is organized according to the following schema. 
First, in Section \ref{setup}, we formally describe our learning setting and problem, discuss prior art, and introduce the SDI algorithm.
We outline some key lemmas and ideas used in the proofs in Section \ref{preliminaries}. In the next two sections, we establish our main results:
in Section \ref{linear}, we establish that for linear models, SDI is asymptotically equivalent to SIS (up to logarithmic factors in the sample size) and in Section \ref{general}, we establish the variable ranking and model selection consistency properties of SDI for more general nonparametric models. 
Lastly, we compare the performance of SDI to other well-known variable selection techniques from a theoretical perspective in Section \ref{comparisons} and from an empirical perspective in Section \ref{experiments}.
We conclude with a few remarks in Section \ref{conclusion}. Due to space constraints, we include the full proofs of our main results in Sections \ref{linear} and \ref{general} in the appendix. 

\section{Setup and algorithm} \label{setup}

In this section we introduce notation, formalize the learning setting, and give an explicit layout of our SDI algorithm. At the end of the section, we discuss its complexity and provide several interpretations.

\subsection{Notation} For labeled data $ \{(U_1, V_1), \dots, (U_n, V_n))\} $ drawn from a population distribution $ (U, V) $, we let $ \widehat{\Cov}(U, V) = \frac{1}{n}\sum_{i=1}^n (U_i - \overline U)(V_i - \overline V) $, $ \widehat{\Var}(U) = \frac{1}{n}\sum_{i=1}^n (U_i-\overline U)^2 $, $ \overline U = \frac{1}{n}\sum_{i=1}^n U_i $, and 
$
\widehat\rho\, (U, V) = \frac{ \widehat{\Cov}(U, V)}{\sqrt{\widehat{\Var}(U) \widehat{\Var}(V)}}
$
denote the sample covariance, variance, mean, and Pearson product-moment correlation coefficient respectively. The population level covariance, variance, and correlation are denoted by $ \Cov(U, V) $,  $\Var(U) = \sigma_U^2 $, and $ \rho(U, V) $, respectively. 




\subsection{Learning setting}
Throughout this paper, we operate under a standard regression framework where the statistical model is $Y = g(\bX)+\varepsilon $, the vector of predictor variables is
$\bX = (X_{1},\ldots,X_{p})^\top$, and $\varepsilon$ is statistical noise. While our results are valid for general nonparametric models, for conceptual simplicity, the canonical model class we have in mind is {\it additive models}, i.e., 
\begin{equation} \label{additive}
g(X_{1},\ldots,X_{p}) = g_1(X_1)+\cdots+g_p(X_p) 
\end{equation}
for some univariate component functions $g_1(\cdot),\ldots,g_p(\cdot)$. As is standard with additive modeling \citep[Section 9.1.1]{hastie2009elements}, for identifiability of the components, we assume that the $g_j(X_j)$ have population mean zero for all $j$. This model class strikes a balance between flexibility and learnability—it is more flexible than linear models, but, by giving up on modeling interaction terms, it does not suffer from the curse of dimensionality.
We observe data $ \calD = \{(\bX_1, Y_1), \dots, (\bX_n, Y_n)\} $ with the $ i^{\text{th}}$ sample point $(\bX_i, Y_i) = (X_{i1},\ldots,X_{ip},Y_i)$ drawn independently from the model above. Note that with this notation, $X_{ij}$ are i.i.d. instances of the random variable $X_j$. 
We assume that the regression function $g(\cdot)$ depends only on a small subset of the variables $ \{X_j\}_{j\in\calS} $, which we call {\it relevant variables} with {\it support} $\calS \subset \{1,\ldots,p\}$ and {\it sparsity level} $s = |\calS| \ll p$. Equivalently, $g_j(\cdot)$ is identically zero for the {\it irrelevant variables} $\{X_j\}_{j\in\calS^c}$. 
In this paper, we consider the \emph{variable ranking} problem, defined here as ranking the variables so that the top $ s $ coincide with $\calS$ with high probability. As a corollary, this will enable us to solve the {\it variable selection} problem, namely, determining the subset $\calS$. 
We pay special attention to the high-dimensional regime where $ p \gg n$. In fact, in Section \ref{main results} we will provide conditions under which consistent variable selection occurs even when $ p = \exp(o(n)) $.

\subsection{Prior art}
The conventional approach to marginal screening for nonparametric additive models
is to directly estimate either the nonparametric components $ g_j(X_j) $ or
the marginal projections
$$ 
f_j(X_j) \coloneqq \mathbb{E}[Y | X_j],
$$ 
with the ultimate goal of studying their variances 
or their correlations with the response variable.\footnote{Note that $ f_j(X_j) $ need not be the same as $g_j(X_j)$ unless, for instance, the predictor variables are independent and the noise is independent and mean zero.} To accomplish this, SIS, NIS, and \citep{hall2009correlation} rank the variables according to the correlations between the response values and least squares fits over a univariate model class $ \calH $, i.e.,
\begin{equation} \label{emp marginal}
\widehat\rho\,(\hat h(X_j), Y), \quad \text{where} \quad \hat h(\cdot) \in \argmin_{h(\cdot) \in \calH}\frac{1}{n}\sum_{i=1}^n (Y_i - h(X_{ij}))^2.
\end{equation}
The model class $ \calH $ is chosen to make the above optimization tractable, while at the same time, be sufficiently rich in order to approximate $ f_j(X_j) $. For example, if $ \calH $ is the space of polynomial splines of a fixed degree, then $ \hat h(\cdot) $ in \eqref{emp marginal} can be computed efficiently via a truncated B-spline basis expansion
\begin{equation} \label{eq:spline}
\beta_1 \Psi_1(X_j) + \beta_2 \Psi_2(X_j) + \cdots + \beta_{d_n}\Psi_{d_n}(X_j),
\end{equation}
as is done with NIS. Similarly, SIS takes $ \calH $ to be the family of linear functions in a single variable.
Complementary methods that aim to directly estimate each $ g_j(X_j) $ include SPAM, which uses a smooth back-fitting algorithm with soft-thresholding, and \citep{huang2010variable}, which combines adaptive group Lasso with truncated B-spline basis expansions.

As we shall see, SDI is equivalent to ranking the variables according to \eqref{emp marginal} when $ \calH $ consists of the collection of all decision stumps in $ X_j $ of the form
\begin{equation} \label{eq:stumps} 
\beta_1 \,\mathbf{1}(X_j \leq z)  + \beta_2 \, \mathbf{1}(X_j > z), \qquad z, \beta_1, \beta_2 \in \mathbb{R}.
\end{equation}
Unlike the previous expressive models such a polynomial splines \eqref{eq:spline}, a one-level decision tree, realized by the model \eqref{eq:stumps} above, severely underfits the data and would therefore be ill-advised for estimating $ f_j(X_j) $, if that were the goal. Remarkably, we show that this rigidity does not hinder SDI for variable selection. What redeems SDI 
is that, unlike the aforementioned methods that are based on linear estimators, decision stumps \eqref{eq:stumps} are \emph{nonlinear} since the splits points can depend on the response data. These model nonlinearities equip SDI with the ability to discover nonlinear patterns in the data, despite its poor approximation capabilities.

Finally, we mention that one benefit of such a simple model as \eqref{eq:stumps} is that it is completely free of tuning parameters. In contrast, other methods such as the ones listed here require careful calibration of, for example, variable bandwidth smoothers or the number of terms in the basis expansions (e.g., SPAM and NIS).

\subsection{The SDI algorithm} \label{sdi}
In this section, we provide the details for the SDI algorithm. We first provide some high-level intuition.

In order to determine whether, say, $X_j$ is relevant for predicting $ Y $ from $ \bX $, it is natural to first divide the data into two groups according to whether $X_j$ is above or below some predetermined cutoff value and then assess how much the variance in $Y$ changes before and after this division. A small change in the variability indicates a weak or nonexistent dependence of $Y$ on $ X_j $; whereas, a moderate to large change indicates heterogeneity in $Y$ across different values of $X_j$. As we now explain, this is precisely what SDI does when the predetermined cutoff value is sought by a least squares fit over all possible ways of dividing the data.

Let $z$ be a candidate split for a variable $X_j$ that divides the response data $ Y $ into left and right daughter nodes based on the $j^{\text{th}}$ variable. Define the mean of the left daughter node to be $ \overline Y_L = \frac{1}{N_L}\sum_{i: X_{ij} \leq z} Y_i $ and the mean of the right daughter node to be $ \overline Y_R = \frac{1}{N_R}\sum_{i: X_{ij} > z}Y_i $ and let the size of the left and right daughter nodes be $ N_L = \#\{i : X_{ij} \leq z\} $ and $ N_R = \#\{i : X_{ij}>  z\} $, respectively. For CART regression trees, the \emph{impurity reduction} (or variance reduction) in the response variable $Y$ from choosing the split point $z$ for the $j^{\text{th}}$ variable is defined to be
\begin{equation} \label{impurity}
\widehat\Delta(z; X_j,Y) = \frac{1}{n}\sum_{i=1}^n (Y_i - \overline Y)^2 -  \frac{1}{n}\sum_{i: X_{ij} \le z}(Y_i - \overline Y_L )^2 - \frac{1}{n}\sum_{i: X_{ij} > z}(Y_i - \overline Y_R )^2.
\end{equation}
For each variable $X_j$, we choose a split point $\hat z_j$ that maximizes the impurity reduction
$$
\hat z_j \in \argmax_z \widehat\Delta(z; X_j, Y),
$$
and for convenience, we denote the largest impurity reduction by
$$
\widehat \Delta(X_j, Y) \coloneqq \widehat\Delta(\hat z_j; X_j, Y).
\footnote{The impurity reduction can be highly non-concave 
	and therefore the optimal split point need not be unique. In such cases, we break ties arbitrarily.}
$$
We then rank the variables commensurate with the sizes of their impurity reductions, i.e., we obtain a ranking $ (\hat{\jmath}_1, \dots, \hat{\jmath}_p) $ where $ \widehat \Delta(X_{\hat{\jmath}_1},Y) \geq \cdots \geq \widehat \Delta(X_{\hat{\jmath}_p},Y)  $. If desired, these rankings can be repurposed to perform model selection (e.g., an estimate $ \widehat \calS $ of $ \calS $), as we now explain.
If we are given the sparsity level $ s $ in advance, we can choose $ \widehat\calS $ to be the top $ s $ of these ranked variables; otherwise, we must find a data-driven choice of how many variables to include.
Equivalently, the latter case is realized by choosing $\mathcal{\widehat{S}}$ to be the indices $j$ for which $\widehat \Delta(X_j,Y) \geq \gamma_n$, where $\gamma_n $ is a threshold to be described in Section \ref{gamma_n}. This is of course a delicate task as including too many variables may lead to more false positives.

By \citep[Section 9.3]{breiman1984}, using a sum of squares decomposition, we can rewrite the impurity reduction \eqref{impurity} as
\begin{equation} \label{LR}
\widehat\Delta(z;X_j, Y) = \frac{N_L}{n}\frac{N_R}{n} (\overline{Y}_L - \overline{Y}_R)^2,
\end{equation}
which allows us to compute the largest impurity reductions for all possible split points with a single pass over the data by first ordering the data along $ X_j $ and then updating $\overline{Y}_L $ and $\overline{Y}_R$ in an online fashion. This alternative expression for the objective function facilitates its rapid evaluation and \emph{exact} optimization. Pseudocode for SDI is given in Algorithm \ref{alg:SDI}.


\begin{algorithm}[H]
	\SetAlgoLined
	\caption{Single-level Decrease in Impurity (SDI)}
	\label{alg:SDI}
	\KwIn{Dataset $\mathcal{D} = \{(X_{i1},\ldots,X_{ip},Y_i)\}_{i=1}^n$}
	\DontPrintSemicolon
	\For{$j=1,\ldots,p$}{
		Relabel $\mathcal{D}$ with $ X_{ij} $ sorted in increasing order\;
		Initialize $ \overline Y_L = 0 $, $ \overline Y_R = \overline Y $, $ \widehat \Delta(X_j,Y) = 0 $\;
		\For{$i=1,\ldots,n-1$}{
			Update $ \overline Y_L \leftarrow \frac{i-1}{i}\overline Y_L + \frac{Y_{i}}{i}, \quad \overline Y_R \leftarrow \frac{n-i+1}{n-i}\overline Y_R - \frac{Y_{i}}{n-i} $ \;
			Compute $\widehat \Delta(X_{ij}; X_j,Y) = \frac{i}{n}(1-\frac{i}{n})(\overline Y_L - \overline Y_R)^2$\;
			\If{$\widehat\Delta(X_{ij}; X_j,Y) > \widehat \Delta(X_j,Y)$}{
			Update $ \widehat \Delta(X_j,Y) \leftarrow \widehat\Delta(X_{ij}; X_j,Y) $
			}
		}
	}
\KwOut{Ranking $ (\hat{\jmath}_1, \dots, \hat{\jmath}_p) $ such that $ \widehat \Delta(X_{\hat{\jmath}_1},Y) \geq \cdots \geq \widehat \Delta(X_{\hat{\jmath}_p},Y)  $}
\end{algorithm}


\subsection{Data-driven choices of $\gamma_n$} \label{gamma_n}


As briefly mentioned in Section \ref{sdi}, if we do not know the sparsity level $s$ in advance, we can instead use a data-driven threshold $\gamma_n$ to modulate the number of selected variables. Here we propose two data-driven methods to determine the threshold $\gamma_n$.

\paragraph{\it Permutation method} The first thresholding method is similar to the Iterative Nonparametric Independence Screening (INIS) method based on NIS  \citep[Section 4]{fan2011nonparametric}. The first step is to choose a random permutation $\pi : \{1, \dots, n \} \to \{1, \dots, n \}$ of the data to decouple $\bX_i$ from $Y_i$ so that the new dataset $\mathcal{D}^\pi = \{(\bX_{\pi(i)},Y_i)\}$ follows a null model. Then we choose the threshold $\gamma_n$ to be the 
maximum of the impurity reductions
$\widehat\Delta(X_j,Y;\mathcal{D}^\pi)$ 
 over all $ j $
based on the dataset $\mathcal{D}^\pi$. We can also generate $T$ different permutations $\pi$ and take the maximum of $ \widehat\Delta(X_j,Y;\mathcal{D}^\pi) $ over all such permuted datasets to get a more significant threshold, i.e., $\gamma_n=\max_{j,\,\pi} \widehat{\Delta}(X_j,Y; \mathcal{D}^{\pi})$. 
With $ \gamma_n $ selected in this way, SDI will then output the variable indices $\widehat \calS$ consisting of the indices $j$ for which the original impurity reductions $\widehat\Delta(X_j,Y;\mathcal{D})$
are at least $\gamma_n$.
Interestingly, this method has parallels to MDA in that we permute the data values of a given variable and calculate the resulting change in the quality of the fit. 	

\paragraph{\it Elbow method} One problem with the permutation method is that it performs poorly when there is correlation between relevant and irrelevant variables. This is because, in the decoupled dataset $\mathcal{D}^\pi$, there is essentially no correlation between $Y$ and any variable; whereas in the original dataset, $Y$ may be correlated with some of the irrelevant variables.
In this case, the threshold $\gamma_n$ will be too small and the selected set of variables will include many irrelevant variables (false positives). 
An alternative is to employ visual inspection via
the \emph{elbow method}, which is typically used to determine the number of clusters in cluster analysis. Here we plot the largest impurity reductions $\widehat\Delta(X_j, Y)$ in decreasing order and search for an ``elbow'' in the curve. We then let $\widehat \calS$ consist of the variable indices that come before
the elbow. Instead of choosing the cutoff point visually, we can also automate the process by clustering the impurity reductions with, for example, a two-component
Gaussian mixture model. Both of these implementations are generally more robust to correlation between relevant and irrelevant variables than the permutation method. 

We illustrate the empirical performance of both the permutation method and the elbow method in Section \ref{experiments}.

\subsection{Computational issues}
We now briefly discuss the computational complexity of Algorithm \ref{alg:SDI}—or equivalently—the computational complexity of growing a single-level CART decision tree.
For each variable $X_j$, we first sort the input data along $ X_j $ with $ \calO(n\log n) $ operations. We then evaluate the decrease in impurity along $ n $ data points (as done in the nested for-loop of Algorithm \ref{alg:SDI}),
and finally find the maximum among these $ n $ values (as done in the nested if-statement of Algorithm \ref{alg:SDI}), all with $ \calO(n) $ operations. 
Thus, the total number of calculations for all of the $ p $ variables is $ \calO(pn\log(n)) $. 
This is only slightly worse than the complexity of SIS for linear models 
$\calO(pn) $, comparable to NIS based on the complexity of fitting B-splines, and favorable to that of Lasso or stepwise regression
$ \calO(p^3+p^2n) $, especially when $ p $ is large \citep{efron2004lars, hastie2009elements}. While approximate methods like coordinate descent for Lasso can reduce the complexity to $ \calO(pn) $ at each iteration, their convergence properties are unclear. Thus, as SPAM is a generalization of Lasso for nonparametric additive models, its implementation (via a functional version of coordinate descent for Lasso) may be similarly expensive.


\subsection{Interpretations of SDI} \label{interpret}

In this section we outline two interpretations of SDI. 

\paragraph{\textit{Interpretation 1}} Our first interpretation of SDI is in terms of the sample correlation between the response and a decision stump. To see this, denote the decision stump that splits $X_j $ at $ z $ by 
$$ \widetilde Y(X_j) \coloneqq \overline Y_L \,\mathbf{1}(X_j \leq z)  + \overline Y_R \, \mathbf{1}(X_j > z) $$ 
and one at an optimal split value $ \hat z_j $ by 
$$\widehat Y(X_j) \coloneqq \overline Y_L \,\mathbf{1}(X_j \leq \hat z_j)  + \overline Y_R \, \mathbf{1}(X_j > \hat z_j).$$ 
Note that $ \widehat Y(X_j) $ equivalently minimizes the marginal sum of squares \eqref{emp marginal} over the collection of all decision stumps \eqref{eq:stumps}.
Next, by Lemma A.1 in \citep{klusowski2020sparse}, we have:
\begin{equation} \label{impurity corr}
\begin{aligned}
\widehat \Delta(z, X_j, Y) & =  \widehat{\Var}(Y) \times \widehat\rho^{\,2}(\widetilde Y(X_j), Y), \quad \text{and}\\
 \widehat\rho^{\,2}(\widetilde Y(X_j), Y) & = 1 -  \frac{\frac{1}{n}\sum_{i=1}^n(Y_i - \widetilde Y(X_{ij}) )^2}{\frac{1}{n}\sum_{i=1}^n (Y_i - \overline Y)^2},
 \end{aligned}
\end{equation}
where 
$$
\widehat\rho\,(\widetilde Y(X_j), Y) \coloneqq \frac{\frac{1}{n} \sum_{i=1}^n (\widetilde Y(X_{ij}) - \overline Y)(Y_i - \overline{Y})}{\sqrt{\frac{1}{n} \sum_{i=1}^n (\widetilde Y(X_{ij}) - \overline Y)^2 \times \frac{1}{n} \sum_{i=1}^n (Y_i - \overline Y)^2 } } \geq 0
$$
is the Pearson product-moment sample correlation coefficient between the data $ Y $ and decision stump $ \widetilde Y(X_j)$. In other words, we see from \eqref{impurity corr} that an optimal split point $ \hat z_j $ is chosen to maximize the Pearson sample correlation between the data $ Y $ and decision stump $ \widetilde Y(X_j)$. This reveals that SDI is, at its heart, a correlation ranking method, in the same spirit as SIS, NIS, and \citep{hall2009correlation} via \eqref{emp marginal}. In fact, as we shall see in Section \ref{preliminaries}, SDI is for all intents and purposes also similar to ranking the variables according to the correlation between the response data and
the marginal projections $ f_j(X_j) $.

\sloppy
Like $r^2$ for linear models, \eqref{impurity corr} reveals that the squared sample correlation $ \widehat\rho^{\,2}(\widetilde Y(X_j), Y) $ equals the \emph{coefficient of determination} $ R^2 $, i.e., the fraction of variance in $ Y $ explained by a decision stump $ \widetilde Y(X_j) $ in $ X_j $.\footnote{However, unlike linear models, for this relationship to be true, the decision stump $ \widetilde Y(X_j) $ need not necessarily be a least squares fit, i.e., $ \widehat Y(X_j) $.} 
Thus, SDI is also equivalent to ranking the variables according to the goodness-of-fit for decision stumps of each variable. In fact, the equivalence between ranking with correlations and ranking with squared error goodness-of-fit is a ubiquitous trait among most models \citep[Theorem 1]{hall2009correlation}. 
\paragraph{\textit{Interpretation 2}}   The other interpretation is in terms of the aforementioned MDI importance measure.
Recall the definition of MDI in Section \ref{mdi}, i.e.,
for an individual decision tree $ T $, the MDI for $ X_j $ is the total reduction in impurity attributed to the splitting variable $ X_j $.
More succinctly, the MDI of $ T $ for $ X_j $ equals 
\begin{equation}\label{mdi} \sum_{\bt}\frac{N(\bt)}{n} \widehat\Delta(X_j, Y | \bX \in \bt),
\end{equation}
where the sum extends over all non-terminal nodes $ \bt $ in which $X_j $ was split, $ N(\bt) $ is the number of sample points in $ \bt $, and $ \widehat\Delta(X_j, Y | \bX \in \bt) $ is the largest reduction in impurity for samples in $ \bt $. Note that if $ T $ is a decision stump with split along $ X_j $, then \eqref{mdi} equals $ \widehat\Delta(X_j, Y) $, the largest reduction in impurity at the root node. Because a split at the root node captures the main effects of the model, $ \widehat\Delta(X_j, Y) $ can be seen as a first order approximation of \eqref{mdi} in which higher order interaction effects are ignored.

\section{Preliminaries} \label{preliminaries}

Our first lemma, developed by the first author in recent work, reveals the crucial role that optimization (of a nonlinear model) plays in assessing whether a particular variable is relevant or irrelevant—by relating the impurity reduction for a particular variable $X_j$ to the sample correlation between the response variable $ Y $ and {\it any} function of $ X_j $. This lemma also highlights a key departure from other approaches in past decision tree literature
that do not consider splits that depend on \emph{both} input and output data (see, for example, DSTUMP \citep{kazemitabar2017variable}).


In order to state the lemma, we will need to introduce the concept of stationary intervals. 
We define a \emph{stationary interval} of a univariate function $ h(\cdot) $ to be a maximal interval $I$ such that $h(I) = c$, where $c$ is a local extremum of $h(\cdot)$ ($ I $ is maximal in the sense that there does not exist an interval $I'$ such that $I \subset I'$ and $h(I') = c$). 
In particular, note that a monotone function does not have any stationary intervals. 

\begin{lemma}[Lemma A.4, Supplementary Material in \citep{klusowski2020sparse}]
	\label{corr}
	Almost surely, uniformly over all functions $ h(\cdot) $ of $ X_j $ that have at most $ M $ stationary intervals, we have
	\begin{equation} \label{cor lower}
	\widehat\Delta(X_j, Y) \geq \frac{1}{D^{-1}Mn + \log(2n)+1} \times\widehat{\Cov}^2\Bigg(\frac{ h(X_j)}{\sqrt{\widehat{\Var}(h(X_j))}}, Y\Bigg),
	\end{equation}
	where $ D \geq 1 $ is the smallest number of data points in a stationary interval of $ h(\cdot) $ that contains at least one data point.\footnote{More precisely, if $ I_1, \dots, I_M $ are the stationary intervals of $ h(\cdot) $ and $ D_k = \#\{X_{ij} \in I_k \} $, then $ D = \min_k \{ D_k : D_k \geq 1 \} $.}
\begin{remark}
Note that $ M $ can also be thought of as the number of times $ h(\cdot) $ changes from strictly increasing to decreasing (or vice versa).
\end{remark}
\begin{remark}
The bound \eqref{cor lower} is tight (up to universal constant factors) when $ M = 0 $, since $ \widehat\Delta(X_j, Y) = 1 $ and $ \widehat{\text{Var}}(Y) \asymp \log(n) $ when  $ X_{1j} \leq \cdots \leq X_{nj} $, $ h(X_{ij}) = Y_i $, and
$$ Y_i = \sqrt{(i-1)(n-i+1)}-\sqrt{i(n-i)}. $$
\end{remark}
\end{lemma}
\begin{proof}[Proof sketch of Lemma \ref{corr}] 
For self-containment, we sketch the proof when $ h(\cdot) $ is differentiable. The essential idea is to construct an empirical prior $ \Pi $ on the split points $z$ and lower bound $ \widehat\Delta(X_j, Y) $ by
	$$
	\int \widehat\Delta(z; X_j, Y)d\Pi(z).
	$$
Recall from Section \ref{sdi} that $ N_L = N_L(z) $ and $ N_R = N_R(z) $ are the number of samples in the left and right daughter nodes, respectively, if the $j^{\text{th}}$ variable is split at $ z $. The special prior we choose has density
	$$
	\frac{d\Pi(z)}{dz} = \frac{|h'(z)|\sqrt{N_L(z)N_R(z)}}{\int |h'(z')|\sqrt{N_L(z')N_R(z')}dz'},
	$$
with support between the minimum and maximum values of the data $ \{X_{ij}\} $.
This then yields $ \widehat\Delta(X_j, Y) \geq C(h) \times \widehat{\Cov}^2\bigg(\frac{ h(X_j)}{\sqrt{\widehat{\Var}(h(X_j))}}, Y\bigg) $. The factor $ C(h) $ can be minimized (by solving a simple quadratic program) over all functions $ h(\cdot) $ 
under the constraint that they have at most $M$ stationary intervals containing at least $ D $ data points,
yielding the desired result \eqref{cor lower}.
	We direct the reader to \citep[Lemma A.4, Supplementary Material]{klusowski2020sparse} 
	for the full proof.
\end{proof}
Ignoring the factor $ (D^{-1}Mn + \log(2n)+1)^{-1} $ in \eqref{cor lower} and focusing only on the squared sample covariance term,
note that choosing 
$h(\cdot) $ to be the marginal projection $ f_j(\cdot)$, we have 
$$
\widehat{\Cov}^2\Bigg(\frac{ f_j(X_j)}{\sqrt{\widehat{\Var}(f_j(X_j))}}, Y\Bigg) \approx \Cov^2\Bigg(\frac{ f_j(X_j)}{\sqrt{\Var(f_j(X_j))}}, Y\Bigg) = \Var(f_j(X_j)),
$$

where the last equality can be deduced from the fact that the marginal projection $ f_j(X_j) $ is orthogonal to the residual $ Y - f_j(X_j) $.
Thus, in an ideal setting, Lemma \ref{corr} enables us to asymptotically lower bound $ \widehat\Delta(X_j, Y) $ by a multiple of the variance of the marginal projections—which can then be used to screen for important variables and control the number of false negatives.



To summarize, the previous lemma shows that $ \widehat\Delta(X_j, Y) $ is large for variables $ X_j $ such that $ f_j(X_j) $ is strongly correlated with $ Y $—or equivalently—variables that have large signals in terms of the variance of the marginal projection. Conversely, our next lemma shows that $ \widehat\Delta(X_j, Y) $ is with high probability not greater than the variance of the marginal projection. 
A special instance of this lemma, namely, when $ Y $ is independent of $ X_j $, was stated in \citep[Lemma 1]{li2019debiased} and serves as the inspiration for our proof. 
 Due to space constraints, we include the proof in Appendix \ref{new-delta-tail}.



\begin{lemma}\label{delta-tail}
	Suppose that $Z_j = Y-f_j(X_j)$ is conditionally sub-Gaussian given $ X_j $, with variance parameter $\sigma_{Z_j}^2$, i.e., $ \mathbb{E}[\exp(\lambda Z_j) | X_j] \leq \exp(\lambda^2\sigma^2_{Z_j}/2) $ for all $ \lambda \in \mathbb{R} $. With probability at least $ 1- 4n\exp(-n\xi^2/(12\sigma_{Z_j}^2)) $, 
$$
\widehat \Delta(X_j, Y) \leq 3\widehat\Var(f_j(X_j)) + \xi^2.
$$
\end{lemma}

In other words, Lemmas \ref{corr} and \ref{delta-tail} together imply that SDI is a proxy for the variance of the marginal projection and therefore it roughly ranks the variables accordingly, up to constant factors. 
These lemmas are a key ingredient of the proofs for model selection and may be of independent interest.

\section{SDI for linear models} \label{linear}
To connect SDI to other variable screening methods that are perhaps more familiar to the reader, we first consider a linear model with Gaussian distributed variables. We allow for any correlation structure between covariates. Recall from \eqref{impurity corr} that $ \widehat \Delta(X_j,Y)$ is equal to $ \widehat\Var(Y) $ times $\widehat\rho^{\,2}(\widehat Y(X_j),Y)$, so that SDI is equivalent to ranking by $\widehat\rho\,(\widehat Y(X_j),Y)$.
 Our first theorem shows that $ \widehat\rho\,(\widehat Y(X_j), Y) $, the sample correlation between $ Y $ and an optimal decision stump in $ X_j $, behaves roughly like the correlation between a linear model $ Y $ and a coordinate $ X_j $. 

\begin{theorem}[SDI is asymptotically equivalent to SIS] \label{sis}
	Let $ Y = \beta_1X_1 + \beta_2X_2 + \cdots + \beta_p X_p + \varepsilon $ 
	and assume that $\bX \sim \mathcal{N}(\bzero, \bSigma) $ for some positive semi-definite matrix $ \bSigma $ and $ \varepsilon \sim \mathcal{N}(0, \sigma^2) $ for some $ \sigma^2 > 0 $. Let $ \delta \in (0, 1) $. 
	 There exists a universal positive constant $C_0$ such that, with probability at least $ 1 - \frac{C_0}{\sqrt{n \delta^2 \rho^2(X_j,Y)}} \exp(-n\delta^2\rho^2(X_j,Y)/2)$,
	\begin{equation}\label{sis-lower}
	\widehat\rho\,(\widehat Y(X_j),Y) \ge \frac{(1-\delta)|\rho(X_j,Y)|}{\sqrt{\log(2n)+1}}.
	\end{equation}
        Furthermore, with probability at least $ 1 - 4n\exp(-n\delta^2/12) - 2\exp(-(n-1)/16)$,
	\begin{equation}\label{sis-upper}
	\widehat\rho\,(\widehat Y(X_j),Y)  \le 5 |\rho(X_j,Y)| + 2\delta.
	\end{equation}
\end{theorem}
\begin{proof}[Proof sketch of Theorem \ref{sis}]
	We only sketch the proof due to space constraints, but a more complete version is provided in Appendix \ref{correlated predictors}.
	
	The first step in proving the lower bound \eqref{sis-lower} is to apply Lemma \ref{corr} with $h(X_j) = X_j$ (a monotone function) to see that
\begin{equation} 
\label{high-prob}
\widehat\Delta(X_j, Y) \geq \frac{\widehat \Var(Y)}{\log(2n)+1} \times\widehat{\rho}^{\;2}(X_j, Y),
\end{equation}
since $ M = 0 $.
Next, we can apply asymptotic tail bounds for Pearson's sample correlation coefficient $ \widehat{\rho}\,(X_j, Y) $ between two correlated Gaussian distributions \citep{hotelling1953correlation} to show that with high probability, $|\widehat{\rho}\,(X_j, Y)| \ge (1-\delta) |\rho(X_j,Y)|$. Finally, we divide \eqref{high-prob} by $ \widehat \Var(Y) $, use \eqref{impurity corr}, and take square roots
to complete the proof of the high probability lower bound \eqref{sis-lower}.

To prove the upper bound \eqref{sis-upper}, notice that since $ X_j $ and $Y$ are jointly Gaussian with mean zero, we have $ f_j(X_j) = \rho_j\frac{\sigma_Y}{\sigma_{X_j}}X_j $, where $ \rho_j = \rho(X_j, Y) $. Thus, by Lemma \ref{delta-tail} with $ \sigma^2_{Z_j} = (1-\rho^2_j)\sigma^2_Y $ and $ \xi^2 = (1-\rho^2_j)\sigma^2_Y \delta^2$, with probability at least $ 1- 4n\exp(-n\delta^2/12) $,
\begin{equation}\label{Delta-additive}
\begin{aligned}
\widehat \Delta(X_j, Y) & \leq 3\widehat\Var(f_j(X_j)) + \delta^2(1-\rho^2_j)\sigma_Y^2 \\ & = 3\rho^2_j\frac{\sigma^2_Y}{\sigma^2_{X_j}}\widehat\Var(X_j) + \delta^2(1-\rho^2_j)\sigma_Y^2.
\end{aligned}
\end{equation}

We further upper bound \eqref{Delta-additive} by obtaining high probability upper and lower bounds, respectively, for $ \widehat{\Var}(X_j)$ and $\widehat{\Var}(Y)$ in terms of $ \sigma^2_{X_j} $ and $ \sigma^2_Y $, with a standard chi-squared concentration bound, per the Gaussian assumption. This yields that 
with high probability,
\begin{equation}\label{Delta-additive2}
\widehat \Delta(X_j, Y) \lesssim \rho^2_j\widehat{\Var}(Y) + \delta^2\widehat{\Var}(Y).
\end{equation}
Finally, dividing both sides of \eqref{Delta-additive2}
by $\widehat{\Var}(Y)$, using \eqref{impurity corr}, and taking square roots proves \eqref{sis-upper}.
\end{proof}



Theorem \ref{sis} shows that with high-probability, SDI is asymptotically equivalent (up to logarithmic factors in the sample size) to SIS for linear models in that it ranks the magnitudes 
 of the marginal 
sample correlations between a variable and the model, i.e., $\widehat\rho\,(X_j, Y) \approx \rho(X_j, Y) $. As a further parallel with decision stumps (see Section \ref{interpret}), the square of the sample correlation, $ \widehat\rho^{\,2}(X_j, Y) $, is also equal to the coefficient of determination $r^2 $ for the least squares linear fit of $ Y $ on $ X_j $. 
We confirm the similarity between SDI and SIS empirically in Section \ref{experiments}.

One corollary of Theorem \ref{sis} is that, like SIS, SDI also enjoys the sure screening property, under the same assumptions as \citep[Conditions 1-4]{fan2008sis}, which include mild conditions on the eigenvalues of the design covariance matrices and minimum signals of the parameters $\beta_j$. 
Similarly, like SIS, SDI can also be paired with lower dimensional variable selection methods such as Lasso or SCAD \citep{antoniadis2001scad} for a complete variable selection algorithm  in the correlated linear model case. 

On the other hand, SDI, a nonlinear method, applies to broader contexts far beyond the rigidity of linear models. 
In the next section, we will investigate how SDI performs for general nonparametric models with additional assumptions on the distribution of the variables.

\section{SDI for nonparametric models} \label{general}

In this section, we establish the variable ranking and selection consistency properties of SDI for general nonparametric models; that is, we show that for Algorithm \ref{alg:SDI}, we have $\BP(\widehat{\calS} = \mathcal{\calS}) \to 1$ as $n\to \infty$. We describe the assumptions needed in Section \ref{assumptions} and outline the main consistency results and their proofs in Section \ref{main results}. Finally, we describe how to modify the main results for binary classification in Section \ref{classification}.

Although our approach differs substantively, to facilitate easy comparisons with other marginal screening methods, our framework and assumptions will be similar. 
As mentioned earlier, SDI is based on a more parsimonious but significantly more biased model fit than those than underpin conventional methods. As we shall see, despite the decision stump severely underfitting the data, SDI nevertheless achieves model selection guarantees that are similar to, and in some cases stronger than, its competitors. This highlights a key difference between quantifying sensitivity and screening—in the latter case, we are not concerned with obtaining {\it consistent} estimates 
of the marginal projections $f_j(X_j)$ and their variances. Doing so demands more from the data and is therefore less efficient, when otherwise crude estimates would work equally well. 

\subsection{Assumptions} \label{assumptions}

In this section, we describe the key assumptions and ideas which will be needed to achieve model selection consistency. The assumptions will be similar to those 
in the independence screening literature \citep{fan2011nonparametric,fan2008sis}, 
but are weaker than most past work on tree based variable selection \cite{li2019debiased,kazemitabar2017variable}.
\begin{itemize}[before=\vspace{-0.5cm}, after=\vspace{1mm}]
\setlength\itemsep{-1.5em}
\item[] \begin{assumption}[Bounded regression function] \label{assumption additive}
The regression function $ g(\cdot) $ is
bounded with $ B = \|g\|_{\infty} < \infty $. 
\end{assumption}

\item[] \begin{assumption}[Smoothness of marginal projections]\label{assumption lipschitz}
Let $ r $ be a positive integer, let $ 0 < \alpha \leq 1 $ be such that $d \coloneqq r+\alpha $ is at least $5/8$, and let $ 0 < L < \infty $. The $ r^{\rm{th}} $ order derivative of $f_j(\cdot) $ exists and is $L$-Lipschitz of order $ \alpha $, i.e.,
$$
\big|f^{(r)}_j(x)-f^{(r)}_j(x')\big| \leq L|x-x'|^{\alpha}, \quad x, x' \in \BR.
$$
\end{assumption}

\item[] \begin{assumption}[Monotonicity of marginal projections] \label{assumption monotone}
The marginal projection $ f_j(\cdot) $ is monotone on $ \BR $. 
\end{assumption}

\item[] \begin{assumption}[Partial orthogonality of predictor variables] \label{assumption orthogonality}
\indent
The collections $ \{ X_j \}_{j \in \calS} $ and $ \{ X_j \}_{j \in \calS^c} $ are independent of each other.
\end{assumption}

\item[] \begin{assumption}[Uniform relevant variables] \label{assumption uniform}
The marginal distribution of each $ X_j $, for $ j \in \calS $, is uniform on the unit interval. 
\end{assumption}


\item[] \begin{assumption}[Sub-Gaussian error distribution] \label{assumption epsilon}
The error distribution is conditionally sub-Gaussian given $ \bX $, i.e., $ \mathbb{E}[\varepsilon | \bX] = 0 $ and $ \mathbb{E}[\exp(\lambda \varepsilon)\mid \bX] \leq \exp(\lambda^2\sigma^2/2) $ for all $ \lambda \in \mathbb{R} $ with $ \sigma^2 > 0 $.
\end{assumption}

\end{itemize}




\subsection{Discussion of the assumptions}

Assumption \ref{assumption lipschitz} 
is a standard smoothness assumption for variable selection in nonparametric additive models \citep[Assumption A]{fan2011nonparametric} and \citep[Section 3]{huang2010variable} except that, for technical reasons, we have the condition $d \ge 5/8$ instead of $d>1/2$. Because SDI does not involve tuning parameters that govern its approximation properties of the nonparametric constituents (such as with NIS and SPAM), Assumption \ref{assumption lipschitz} can be relaxed to allow for different levels of smoothness in different dimensions and, by straightforward modifications of our proofs, one can show that SDI adapts automatically.
Alternatively, instead of Assumption \ref{assumption lipschitz}, as we shall see, stronger conclusions can be provided if we impose a monotonicity constraint, namely, Assumption \ref{assumption monotone}. Note that this monotonicity assumption encompasses many important ``shape constrained'' statistical models such as linear or isotonic regression.
Assumption \ref{assumption orthogonality} is essentially the so-called ``partial orthogonality'' condition in marginal screening methods \citep{fan2010sure}. Importantly, it allows for correlation between the relevant variables $ \{ X_j \}_{j \in \calS} $, unlike previous works on tree based variable selection \citep{kazemitabar2017variable,li2019debiased}.
Notably, NIS and SPAM do allow for dependence between relevant and irrelevant variables, under suitable assumptions on the data matrix of basis functions. However, these assumptions are difficult to translate in terms of the joint distribution of the predictor variables and difficult to verify given the data.


Assumption \ref{assumption uniform} is stated as is for clarity of exposition and is not strictly necessary for our main results to hold.
For instance,
we may assume instead that the marginal densities of the relevant variables are compactly supported and uniformly bounded above and below by a strictly positive constant, as in \citep{fan2011nonparametric, huang2010variable}.
In fact, even these assumptions are not required. If the marginal projection is monotone (i.e., Assumption \ref{assumption monotone}), \emph{no marginal distributional assumptions are required}, that is, each $ X_j $ for $ j \in \calS $ could be continuous, discrete, have unbounded support, or have a density that vanishes or is unbounded.
More generally,
similar 
distributional relaxations
are made possible by the fact that CART decision trees are invariant to monotone transformations, enabling us to reduce the general setting to the case where each predictor variable is uniformly distributed on $ [0, 1] $. See Remark \ref{uniform} for details.
	\begin{remark} \label{uniform}
	Let $q_j(\cdot)$ and $F_j(\cdot)$ be the quantile function and the cumulative distribution function of $X_j$, respectively. Recall the Galois inequalities state that $ q_j(w) \leq z $ if and only if $w \leq F_j(z)$ and furthermore that $q_j (F_j(X_j)) = X_j$ almost surely. Then, choosing $w = F_j(X_j)$ we see that, almost surely, $X_j \leq z$ if and only if $F_j(X_j) \le F_j(z)$. Therefore, almost surely,
	$$
	\max_z\widehat \Delta(z;X_j,Y) = \max_z\widehat \Delta(F_j(z);F_j(X_j),Y) = \max_{w\in[0,1]}\widehat \Delta(w;F_j(X_j),Y).
	$$
	This means that for continuous data, the problem can be reduced to the uniform case by pre-applying the marginal cumulative distribution function $ F_j(\cdot) $ to each variable, since $ F_j(X_j) \sim \rm{Uniform}([0, 1]) $. Note that the marginal projections now equal the composition of the original $ f_j(\cdot) $ with $ q_j(\cdot) $, i.e., $ (f_j\circ q_j)(\cdot) $. By the chain rule from calculus, if $ q_j(\cdot) $ satisfies Assumption \ref{assumption lipschitz}, then so does $ (f_j\circ q_j)(\cdot) $.
	\end{remark}

\subsection{Theory for variable ranking and model selection} \label{main results}
Here our goal will be to provide 
 variable ranking and model selection guarantees of SDI using the assumptions in Section \ref{assumptions}. 
Again, in this section, we sketch the proofs, but the full versions can be found in Appendices \ref{23} and \ref{45}.

\subsubsection{Preliminary results}

The high level idea will be to show that the impurity reductions for relevant variables dominate those for irrelevant variables with high probability, 
meaning that relevant and irrelevant variables are correctly ranked.


The following two propositions provide high probability lower bounds on the impurity reduction for relevant variables, the size of which depend on whether we assume Assumption \ref{assumption lipschitz} or Assumption \ref{assumption monotone}.

Our first result deals with general, smooth marginal projections. Remarkably, it shows that with high probability, $ \widehat\Delta(X_j, Y) $ captures a portion of the variance in the marginal projection.

\begin{proposition} \label{proposition lipschitz}
Under Assumptions \ref{assumption additive},
\ref{assumption lipschitz}, \ref{assumption uniform}, and \ref{assumption epsilon},
with probability at least $1-(4n+2) \exp\big(-nC_1\Var(f_j(X_j)) \big)$,
 we have
$$
\widehat\Delta(X_j, Y) \geq \frac{C_2 (\Var(f_j(X_j)))^{6/5+1/d} }{\log(n)},
$$
for some positive constants $C_1$ and $C_2$ which depend only on $B$, $\sigma$, $ r $, and $ \alpha$.
\end{proposition}

Next, we state an analogous bound, but under a slightly different assumption on the marginal projection. It turns out that if the marginal projection is monotone, we can obtain a stronger result, which is surprisingly independent of the smoothness level.

\begin{proposition} \label{proposition monotone}
Under Assumptions  \ref{assumption additive},
\ref{assumption monotone}, and \ref{assumption epsilon},
with probability at least $ 1 - 2\exp\Big(-\frac{(n-1)\Var(f_j(X_j))}{32(B^2+\sigma^2)}\Big) $,
$$
\widehat\Delta(X_j, Y) \geq \frac{\Var(f_j(X_j))}{16(1+\log(2n))}.
$$
\end{proposition} 

\begin{proof}[Proof sketch of Propositions \ref{proposition lipschitz} and \ref{proposition monotone}]
We sketch the proof of Proposition \ref{proposition lipschitz}; the proof of Proposition \ref{proposition monotone} is based on similar arguments. The main idea is to apply Lemma \ref{corr} with $ h(\cdot) $ equal to a modified polynomial approximation $ \tilde f_j(\cdot) $ to $f_j(\cdot)$. This is done to temper the effect of the factor $ (D^{-1}Mn+\log(2n)+1)^{-1} $ from Lemma \ref{corr}, by controlling $ M $ and $ D $ individually.

To construct such a function, we first employ a Jackson-type estimate \citep{jackson1930approximation} in conjunction with Assumption \ref{assumption lipschitz} and Bernstein's theorem for polynomials \citep{jackson1935bernstein} to show the existence of a good polynomial approximation $P_M(\cdot)$ (of degree $ M+1$) to $f_j(\cdot)$.
We then construct $ \widetilde f_j(\cdot) $ by redefining $P_M(\cdot)$ to be constant in a small neighborhood around each of its local extrema, which ensures that each resulting stationary interval of $ \tilde f_j(\cdot) $ has a sufficiently large length. Since $ P_M(\cdot) $ is a polynomial of degree $ M+1 $, it has at most $ M $ local extrema, and thus the number of stationary intervals of $ \tilde f_j(\cdot) $ will also be at most $ M $.

Next, we use concentration of measure to ensure that each stationary interval of $ \tilde f_j(\cdot) $ is saturated with enough data, effectively providing a lower bound on $ D $. Executing this argument reveals that valid choices of $ D $ and $ M $ (which come from optimizing a bound on the supremum norm between $ \tilde f_j(\cdot) $ and $ f_j(\cdot) $) are:
$$
D \gtrsim n \times (\Var(f_j(X_j)))^{1/5+1/(2d)}, \quad M \lesssim (\Var(f_j(X_j)))^{-1/(2d)}.
$$
Plugging these values into the lower bound \eqref{cor lower} in Lemma \ref{corr}, we find that with high probability,
\begin{equation} \label{lower plug}
\begin{aligned}
\widehat\Delta(X_j, Y) & \geq \frac{1}{D^{-1}Mn + \log(2n)+1} \times\widehat{\Cov}^2\Bigg(\frac{ \tilde f_j(X_j)}{\sqrt{\widehat{\Var}(\tilde f_j(X_j))}}, Y\Bigg) \\
& \gtrsim \frac{(\Var(f_j(X_j)))^{1/5+1/d}}{\log(n)} \times\widehat{\Cov}^2\Bigg(\frac{ \tilde f_j(X_j)}{\sqrt{\widehat{\Var}(\tilde f_j(X_j))}}, Y\Bigg).
\end{aligned}
\end{equation}

Thus, the lower bound in Proposition \ref{proposition lipschitz} will follow if we can show that the squared sample covariance factor in \eqref{lower plug} exceeds $ \Var(f_j(X_j)) $ with high probability.
To this end, note that
\begin{equation}\label{cov}
	\begin{aligned}
\widehat{\Cov}\Bigg(\frac{ \tilde f_j(X_j)}{\sqrt{\widehat{\Var}(\tilde f_j(X_j))}}, Y\Bigg) &= \underbrace{\widehat{\Cov}\Bigg(\frac{ \tilde f_j(X_j)}{\sqrt{\widehat{\Var}(\tilde f_j(X_j))}}, f_j(X_j)\Bigg)}_{\text{(I)}} \\
& \qquad+ \underbrace{\widehat{\Cov}\Bigg(\frac{ \tilde f_j(X_j)}{\sqrt{\widehat{\Var}(\tilde f_j(X_j))}}, Y-f_j(X_j)\Bigg)}_{\text{(II)}}.
	\end{aligned}
\end{equation}
With high probability, 
(I) can be lower bounded by $ C\sqrt{\Var(f_j(X_j))} $, where $ C $ is some constant, using the approximation properties of $ \tilde f_j(\cdot) $ for $f_j(\cdot) $, per the choice of $ D $ and $ M $ from above, and a concentration inequality for the sample variance of $ f_j(X_j) $.
Furthermore, since $ Y - f_j(X_j) $ has conditional mean zero,
a Hoeffding type concentration inequality shows that, with high probability, 
(II) is larger than any (strictly) negative constant, including $ -(C/2)\sqrt{\Var(f_j(X_j))} $.
Combining this analysis from \eqref{lower plug} and \eqref{cov}, we obtain the high probability lower bound on $\widehat\Delta(X_j, Y)$ given in Proposition \ref{proposition lipschitz}.
\end{proof}
Next, we need to ensure that there is a sufficient separation in the impurity reductions between relevant and irrelevant variables. 
To do so, we use Lemma \ref{delta-tail} along with the partial orthogonality assumption in Section \ref{assumptions} to show that the impurity reductions for irrelevant variables will be small with high probability.
\begin{lemma} \label{delta-tail-2}
	Under Assumptions \ref{assumption additive}, \ref{assumption orthogonality} and \ref{assumption epsilon}, for each $ j \in \calS^c $, with probability at least $ 1- 4n\exp(-n\xi^2/(12(B^2+\sigma^2))) $,
	$$
	\widehat \Delta(X_j, Y) \leq \xi^2.
	$$
	In other words, if $ j \in \calS^c $, then $ \widehat\Delta(X_j, Y) = \calO(n^{-1}\log(n)) $ with probability at least $ 1 - n^{-\Omega(1)} $.
\end{lemma}
\begin{proof}[Proof of Lemma \ref{delta-tail-2}]
	Observe that
	\begin{align}
	\BE[\exp(\lambda (Y-f_j(X_j))) | X_j] 
	&= \BE\big[\exp(\lambda (g(\bX)-f_j(X_j)))\BE\big[\exp(\lambda\varepsilon)\big | \bX\big]\big | X_j\big] \nonumber\\
	&\leq \BE\big[\exp(\lambda (g(\bX)-f_j(X_j)))\big | X_j\big] \exp(\lambda^2 \sigma^2/2) \label{pen2}\\
	&\leq \BE\big[\exp(\lambda^2 (B^2+\sigma^2)/2) \big],\label{last2}
	\end{align}
	where we used Assumption \ref{assumption epsilon} in the penultimate inequality \eqref{pen2} and Hoeffding's Lemma together with Assumption \ref{assumption additive} in the last inequality \eqref{last2}. 
	Using Assumption \ref{assumption orthogonality} along with Lemma \ref{delta-tail} with $ \sigma^2_{Z_j} = B^2 + \sigma^2 $ proves
	Lemma \ref{delta-tail-2}.
\end{proof} 

\subsubsection{Main results}
Assuming we know the size $s$ of the support $\calS$, we can use the SDI ranking from Algorithm \ref{alg:SDI} to choose the top $s$ most important variables.
Alternatively, if $ s $ is unknown, we instead choose an asymptotic threshold $\gamma_n$ of the impurity reductions to select variables; that is, $\widehat \calS = \{j: \widehat \Delta(X_j,Y) \geq \gamma_n \}$.
We state our variable ranking guarantees in terms of the 
minimum signal strength of the relevant variables:
$$
v \coloneqq \min_{j\in\calS} \Var(f_j(X_j)),
$$
which is the same as the minimum variance parameter in independence screening papers (e.g., \citep[Assumption C]{fan2011nonparametric}). Note that $ v $ measures the minimum contribution of each relevant variable alone to the variance in $Y$, ignoring the effects of the other variables.

\begin{theorem} \label{theorem lipschitz ii}
Suppose Assumptions \ref{assumption additive},
\ref{assumption lipschitz}, \ref{assumption orthogonality}, \ref{assumption uniform}, and \ref{assumption epsilon} hold.
Then the top $ s $ most important variables ranked by Algorithm \ref{alg:SDI} equal the correct set $ \calS $ of relevant variables with probability at least
\begin{equation} \label{eq:prob_gen}
1-s(4n+2)\exp(- C_1 n v) - 4n(p-s)\exp\Big(-\frac{nC_2v^{6/5+1/d}}{24\log(n)(B^2+\sigma^2)}\Big),
\end{equation}
where $C_1$ and $C_2$ are the same constants in Proposition \ref{proposition lipschitz}.
\end{theorem}

\begin{remark}
As a corollary of Theorem \ref{theorem lipschitz ii}, we obtain a special case of Proposition 1 in \citep{scornet2015} for the root node of a CART decision tree, which states more generally that a relevant variable is selected at each node with probability converging to one. 
\end{remark}

As mentioned earlier, stronger results can be obtained if we assume that the marginal projections are monotone. In our next theorem, notice that we do not have to make any additional smoothness assumptions, nor do we require any distributional assumptions on the relevant variables, other than that they are independent of the irrelevant ones.

\begin{theorem} \label{theorem monotone ii}
	Suppose Assumptions \ref{assumption additive},
	\ref{assumption monotone}, \ref{assumption orthogonality}, and \ref{assumption epsilon} hold.
	Then the top $ s $ most important variables ranked by Algorithm \ref{alg:SDI} equal the correct set $ \calS $ of relevant variables with probability at least
	\begin{equation} \label{eq:prob_mon}
	1 - 2s\exp\Big(-\frac{(n-1)v}{32(B^2+\sigma^2)}\Big)- 4n(p-s)\exp\Big(-\frac{nv}{96(1+\log(2n))(B^2+\sigma^2)}\Big).
	\end{equation}

\end{theorem}
 
\begin{remark}
When $s$ is unknown, Propositions \ref{proposition lipschitz} and \ref{proposition monotone} and Lemma \ref{delta-tail-2} imply that the oracle threshold choices
\begin{equation} \label{eq:threshold3}
\gamma_n = \frac{C_2v^{6/5+1/d} }{2\log(n)} \quad \text{and} \quad \frac{v}{8(1+\log(2n))} 
\end{equation}
ensure that 
$$
\max_{j\in\calS^c} \widehat\Delta(X_j, Y) < \gamma_n \leq \min_{j\in\calS} \widehat \Delta(X_j, Y)
$$
and hence will yield the same high probability bounds \eqref{eq:prob_gen} and \eqref{eq:prob_mon}, respectively.
Thus, while the permutation and elbow methods
from Section \ref{gamma_n} are somewhat ad-hoc, if they, at the very least, produce thresholds that are close to \eqref{eq:threshold3},
then high probability performance guarantees are still possible to obtain. 
\end{remark}

\subsection{Minimum signal strengths}
Like all marginal screening methods, the theoretical basis for SDI is that each marginal projection for a relevant variable should be nonconstant, or equivalently, that $ v > 0 $. Note that when the relevant variables are independent and the underlying model is additive, per \eqref{additive}, the marginal projections equal the component functions of the additive model. Hence, $ v = \min_{j\in\calS} \Var(g_j(X_j))  $, which will always be strictly greater than zero.
As Theorems \ref{theorem lipschitz ii} and \ref{theorem monotone ii}  show, $ v $ controls the probability of a successful ranking of the variables. In practice, many of the relevant variables may have very small signals—therefore we are particularly interested in cases where $ v $ is allowed to become small when the sample size grows large, as we now discuss.

We see from Theorem \ref{theorem lipschitz ii} that in order to have model selection consistency with probability at least $ 1 - n^{-\Omega(1)} $,
it suffices to have
\begin{equation} \label{eq:threshold}
v \gtrsim \Big(\frac{\log(n)\log(n(p-s))}{n}\Big)^{\frac{5d}{6d+5}},
\end{equation}
up to constants that depend on $ B $, $ \sigma $, $ r $, and $ \alpha $.
That is, \eqref{eq:threshold} is a sufficient condition on the signal of all relevant variables so that $ \mathbb{P}(\widehat \calS = \calS)  \rightarrow 1 $ as $n\to\infty$. 
Similarly, we see from Theorem \ref{theorem monotone ii} that
\begin{equation} \label{eq:threshold2}
v \gtrsim \frac{\log(n)\log(n(p-s))}{n}
\end{equation}
is sufficient to guarantee
model selection consistency for
monotone marginal projections. A particularly striking aspect of \eqref{eq:threshold2} is that the rate is independent of the smoothness of the marginal projection. This means that the ``difficulty'' of detecting a signal from a general monotone marginal projection is essentially no more than if the marginal projection was linear.

\subsection{SDI for binary classification} \label{classification}


Our theory for SDI can also naturally be extended to the context of classification, as we now describe. For simplicity, we focus on the problem of binary classification where $ Y \in \{0, 1\} $. 
To begin, we first observe that 
Gini impurity and variance impurity are equivalent up to a factor of two \citep[Section 3]{louppe2014understanding}. Thus, we can use the same criterion $ \widehat\Delta(X_j, Y) $ to rank the variables. 
Next, we identify the marginal projection as the marginal class probability
 $$ f_j(X_j) = \mathbb{P}(Y=1 | X_j) = 1 - \mathbb{P}(Y=0 | X_j). $$
Because of these connections to the regression setting, the results in Section \ref{main results} hold verbatim for classification under the same assumptions therein
with $ v = \min_{j\in\calS} \Var(\mathbb{P}(Y=1 | X_j))  $, $ B = 1 $, and $ \sigma^2 = 0 $. 
An interesting special case arises when $ \mathbb{P}(Y = 1 | \bX) = h(\beta_1X_1+\cdots \beta_pX_p) $ for some monotone link function $ h(\cdot) $ and $\bX \sim \mathcal{N}(\bzero, \bSigma) $ for some positive semi-definite matrix $ \bSigma $. Using the fact that Gaussian random vectors are conditionally Gaussian distributed, it can be shown that the marginal projection $ \mathbb{P}(Y = 1 | X_j) $ is monotone and therefore satisfies Assumption \ref{assumption monotone}. Consequently, for example, logistic regression with Gaussian data inherits the same stronger conclusions as Theorem \ref{theorem monotone ii}.



\section{Comparing SDI with other model selection methods} \label{comparisons}


In this section, we compare the finite sample guarantees of SDI given in Section \ref{main results} and Section \ref{linear} to those of NIS, Lasso, and SPAM. To summarize, we find that SDI enjoys model selection consistency even when the marginal signal strengths of the relevant variables are smaller than those for NIS and SPAM. We also find that the minimum sample size of SDI for high probability support recovery is nearly what is required for Lasso, which is minimax optimal. Finally, we show that SDI can handle a larger number of predictor variables than NIS and SPAM.

\paragraph{\textit{Minimum signal strength for NIS}}
	We analyze the details of \citep{fan2011nonparametric} to uncover the corresponding threshold $v$ for NIS.
	In order to have model selection consistency, the probability bound in \citep[Theorem 2]{fan2011nonparametric}
	must approach one as $n\to \infty$, which necessitates
	\begin{equation} \label{d_n}
	d_n \lesssim \Big(\frac{n^{1-4\kappa}}{\log(np)}\Big)^{1/3},
	\end{equation}
where $ d_n $ is the number of spline basis functions and $ \kappa $ is a free parameter (in the notation of \citep{fan2011nonparametric}).
	However notice that by \citep[Assumption F]{fan2011nonparametric}, we must also have that $d_n \gtrsim n^{\frac{2\kappa}{2d+1}}$, and combining this with \eqref{d_n} shows that we must have
	\begin{equation} \label{kappa}
	n^{-\kappa} \gtrsim \Big(\frac{\log(np)}{n}\Big)^{\frac{2d+1}{8d+10}}.
	\end{equation}
	Now substituting \eqref{d_n} and \eqref{kappa} into $v \gtrsim d_n n^{-2\kappa}$ (\citep[Assumption C]{fan2011nonparametric}), it follows that we have
$$
	v \gtrsim\Big(\frac{\log (np)}{n}\Big)^{\frac{4d}{8d+10}}
$$
for NIS,
	which is a larger minimum signal than our \eqref{eq:threshold}.

	\paragraph{\textit{Minimum signal strength for SPAM}}
	In the case where $d = 2$, by \citep[Section 6.1]{ravikumar2009spam}, we must have $v \gtrsim n^{-4/15}\log^{16/5} (np)$ for SPAM to achieve consistent model selection. For comparison, our algorithm allows for a smaller signal $v \gtrsim \big(\frac{\log(n)\log(np)}{n}\big)^{10/17}$, which is obtained by setting $d = 2$ in \eqref{eq:threshold}. 

\paragraph{\textit{Minimum sample size for consistency}}
Consider the linear model with Gaussian variates from Theorem \ref{sis}, where for simplicity we additionally assume that $ \bSigma = \mathbf{I}_{p\times p} $
is the $ p\times p $ identity matrix, yielding $ \rho^2(X_j, Y) = \beta^2_j/(\sigma^2+\sum_{k=1}^p \beta^2_k) $. 
	Following the same steps used to prove Theorem \ref{theorem lipschitz ii} but using Theorem \ref{sis} and Lemma \ref{delta-tail} instead, we can derive a result similar to Theorem \ref{theorem lipschitz ii} for the probability of exact support recovery, 
	but for a linear model with Gaussian variates. The full details are in Appendix \ref{remark}.
	With the specifications $ \sum_{k=1}^p \beta^2_k = \calO(1)$ and $\min_{j\in\calS}|\beta_j|^2 \asymp 1/s$, we find that a sufficient sample size for high probability support recovery is
	$$
	n \gg s\log(n)\log(n(p-s)),
	$$
	which happens
	when
	\begin{equation}  \label{threshold}
	n \gg s\log (p-s) \times (\log (s) + \log \log (p-s)).
	\end{equation}
	Now, it is shown in \citep[Corollary 1]{wainwright2009information} 
	that the minimax optimal threshold for support recovery under these parameter specifications is $n \asymp s \log (p-s)$, which is achieved by Lasso \citep{wainwright2009sharp}. Amazingly, \eqref{threshold} coincides with this optimal threshold up to $\log (s) $ and $ \log \log(p-s)$ factors, despite SDI not being tailored to linear models. 

\paragraph{\textit{Maximum dimensionality}}
Suppose the signal strength $v$ is bounded above and below by a positive constant when the sample size increases. 
Then Theorems \ref{theorem lipschitz ii} and \ref{theorem monotone ii} show model selection consistency for SDI up to dimensionality $ p = \exp(o(n)) $. This is larger than the maximum dimensionality $ p = \exp(o(n^{2(d-1)/(2d+1)})) $ for NIS \citep[Section 3.2]{fan2011nonparametric}, thus applying to an even broader spectrum of ultra high-dimensional problems. Furthermore, when $ d = 2 $, SPAM is able to handle dimensionality up to $ p = \exp(o(n^{1/6})) $ \citep[Equation (45)]{ravikumar2009spam},
which is again lower than the dimensionality $ p = \exp(o(n)) $ for SDI. 
\section{Experiments} \label{experiments}


In this section, we conduct computer experiments of SDI with synthetic data.
As there are many existing empirical studies of 
the related MDI measure
\citep{kazemitabar2017variable,li2019debiased,louppe2014understanding,louppe2013understanding,lundberg2018consistent,strobl2007bias,wang2016experimental,wei2015variable},
we do not aim for comprehensiveness.

Our first set of experiments compare the performance of SDI with SIS, NIS, Lasso, and SPAM, MDI for a single CART decision tree, and MDI for a random forest. 
In Section \ref{exact}, we assess performance based on the probability of {\it exact support recovery}.
To ensure a fair comparison between SDI and the other algorithms, we assume a priori knowledge of the true sparsity level $ s $, which is incorporated
 into Lasso and SPAM by specifying the model degrees of freedom in advance. These simulations 
 were conducted in R using the packages \texttt{rpart} for SDI, \texttt{SAM} for SPAM, \texttt{SIS} for SIS, and \texttt{glmnet} for Lasso with default settings. We also compute two versions of MDI: MDI RF using the package \texttt{randomForest} with \texttt{ntrees = 100} and MDI CART (based on a pruned CART decision tree) using the package \texttt{rpart} with default settings. The source code from \citep{fan2011nonparametric} was used to conduct experiments with NIS.

The second set of experiments, in Section \ref{gamma_n_2}, deals with the case when the sparsity level $ s $ is unknown, whereby we demonstrate the empirical performance of the permutation and elbow methods from Section \ref{gamma_n}.

In all our experiments, we generate $n$ samples from an $s$-sparse additive model
$ g(\bX) =\sum_{j=1}^s g_j(X_j) $ for various types of components $g_j(X_j)$. 
The error distribution is $ \varepsilon \sim \mathcal{N}(0, \sigma^2) $,
the sparsity level is fixed at $s = 4$, and the ambient dimension is fixed at $ p = 2000 $. We consider the following model types.

\paragraph{Model 1}
Consider linear additive components $g_j(X_j) = X_j$ and variables $ \bX \sim \mathcal{N}(\bzero, \bSigma) $, where the covariance matrix $ \bSigma $ has diagonal entries equal to $ 1 $ and off-diagonal entries equal to some constant $ \rho \in (-1, 1) $. We set the noise level $\sigma^2 = 1$ and consider correlation level $\rho = 0.5$.

\paragraph{Model 2} 
Consider linear additive components $g_j(X_j) = X_j$ with noise level $\sigma^2 = 3$. The variables $X_2,\ldots,X_p$ are i.i.d standard normal random variables and $X_1= -\frac{1}{3}X_2^3+ \tilde{\varepsilon}$ where $\tilde{\varepsilon} \sim \calN(0,1)$. Notice that even though the model is linear, the marginal projections $\BE[Y|X_1]$ and $\BE[Y|X_2]$ are nonlinear. This is similar to Example 2 of \citep{fan2011nonparametric}.

\paragraph{Model 3} We consider additive components
$g_j(X_j) = \cos(4\pi X_j)$, where $ \bX \sim \text{Uniform}([0, 1]^p)$ 
(i.e., all predictor variables are independent) and $\sigma^2 = 1$.

\paragraph{Model 4} 
Consider the nonlinear additive components
\begin{align*}
& g_1(x) = 5x, \quad g_2(x) = 3(2x-1)^2, \quad g_3(x) = \frac{4\sin(2\pi x)}{2 - \sin(2 \pi x)},\\
& g_4(x) = 6(0.1\sin(2\pi x) + 0.2\cos(2\pi x)+0.3\sin^2(2\pi x)+0.4\cos^3(2\pi x) \\
& \quad \quad \quad \quad + 0.5\sin^3(2\pi x)).
\end{align*}
Let $ \bX \sim \text{Uniform}([0, 1]^p)$ and set the noise level $\sigma^2 = 1.74$. This is the same model as Example 3 of \citep{fan2011nonparametric} with $t = 0$ (and correlation $0$).



\paragraph{Model 5} 
We consider monotone additive components 
\begin{align*}
& g_1(x) = -\exp(x^2), \quad g_2(x) = -\log(x+0.1), \\
& g_3(x) = 2\tanh(20x^2)+0.5\exp(x^3), \quad g_4(x) = \frac{2\exp(10x - 5)}{1 + \exp(10x-5)}.
\end{align*}
The variables are $ \bX \sim \text{Uniform}([0, 1]^p)$ and we have noise level $\sigma^2 = 1$.
This is similar to Model A in \citep[Section 3]{bergersen2014monotone}.

In our simulations, Models 1 tests the correlated Gaussian linear model setting of Theorem \ref{sis}. Model 2 modifies the setting of Theorem \ref{sis} by including a relevant variable that is a nonlinear function of another relevant variable so that some marginal projections are nonlinear. Models 3 and 4 are examples of the setting of Theorem \ref{theorem lipschitz ii} for general nonparametric models. Model 5 considers the shape-constrained setting of Theorem \ref{theorem monotone ii}.
Though our main results apply to general nonparametric models, we have chosen to focus our experiments on additive models to facilitate comparison with other methods designed for the same setting.

\subsection{Exact support recovery} \label{exact}
For our experiments on exact recovery, we fix the sparisty level $ s = 4 $
and estimate the probability of exact support recovery by running $50$ independent replications and computing the fraction of replications which exactly recover the support of the model.
In Figure \ref{fig:exact}, we plot this estimated probability against various sample sizes $ n < p $, namely, $ n \in \{100, 200, 300, \dots, 1000 \} $.
In agreement with Theorem \ref{sis}, in Figure \ref{fig:1}, we observe that SDI and SIS exhibit similar behavior for correlated Gaussian linear models, a case in which all methods appear to achieve model selection consistency. However, in Figure \ref{fig:2}, we see that SDI is significantly more robust than SIS to a linear model when the marginal projections are nonlinear.
As expected, Figure \ref{fig:3} and \ref{fig:4} show that SDI, NIS, SPAM, and MDI significantly outperform SIS and Lasso when the model has nonlinear and non-monotone additive components.
For more irregular component functions such as sinusoids, SDI appears to 
outperform SPAM, as seen in Figure \ref{fig:3}. In agreement with Theorem \ref{theorem monotone ii}, Figure \ref{fig:5} shows that SDI appears to perform better for nonlinear monotone components, though all methods (even linear methods such as Lasso and SIS) appear to perform better as well. In general, for additive models, SDI appears to outperform MDI CART though it appears to sacrifice a small amount of accuracy for simplicity compared to its progenitor MDI RF.
In summary, in agreement with Theorems \ref{theorem lipschitz ii} and \ref{theorem monotone ii}, Figure \ref{fig:exact} demonstrates the robustness of SDI across various additive models. 

\begin{figure}[H] 
	\centering
	\begin{subfigure}[t]{0.3\textwidth}
	\centering
	\includegraphics[width=1\linewidth]{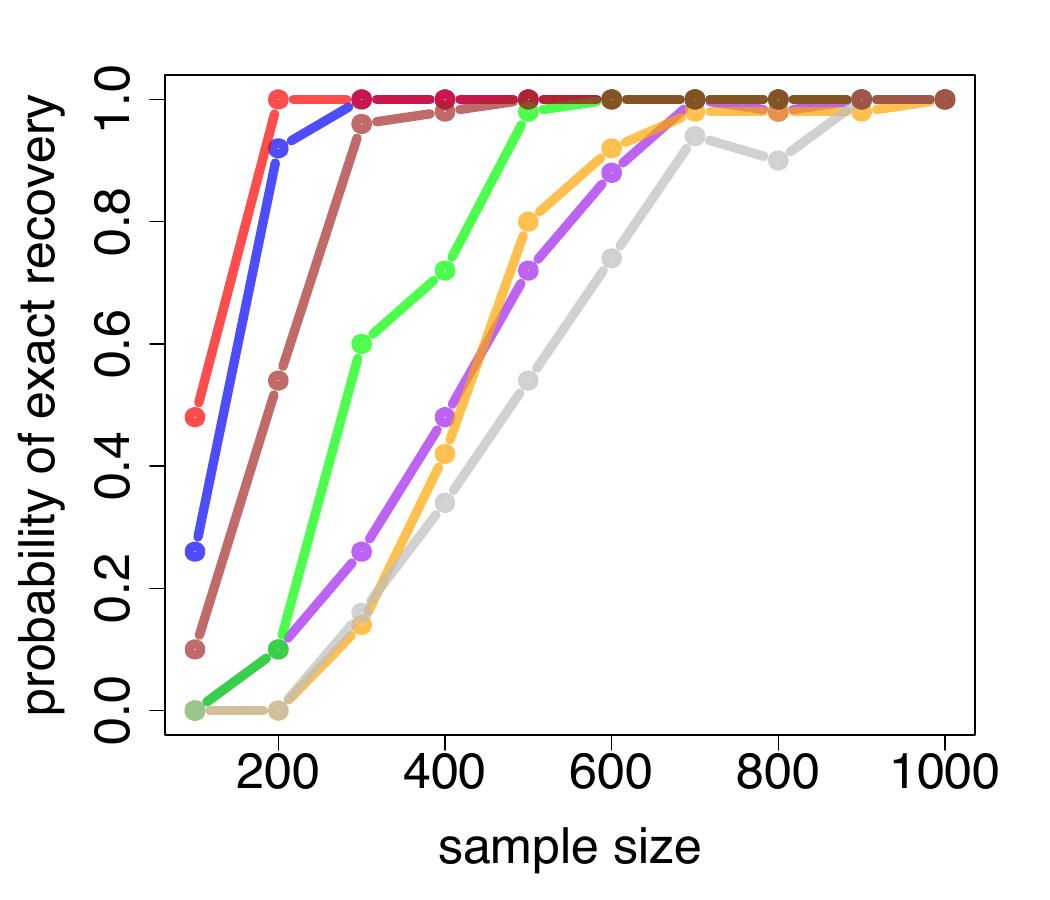}
	\caption{Model 1 (SNR $10.00$)}
			\label{fig:1}
	\end{subfigure}
	\begin{subfigure}[t]{0.3\textwidth}
	\centering
	\includegraphics[width=1\linewidth]{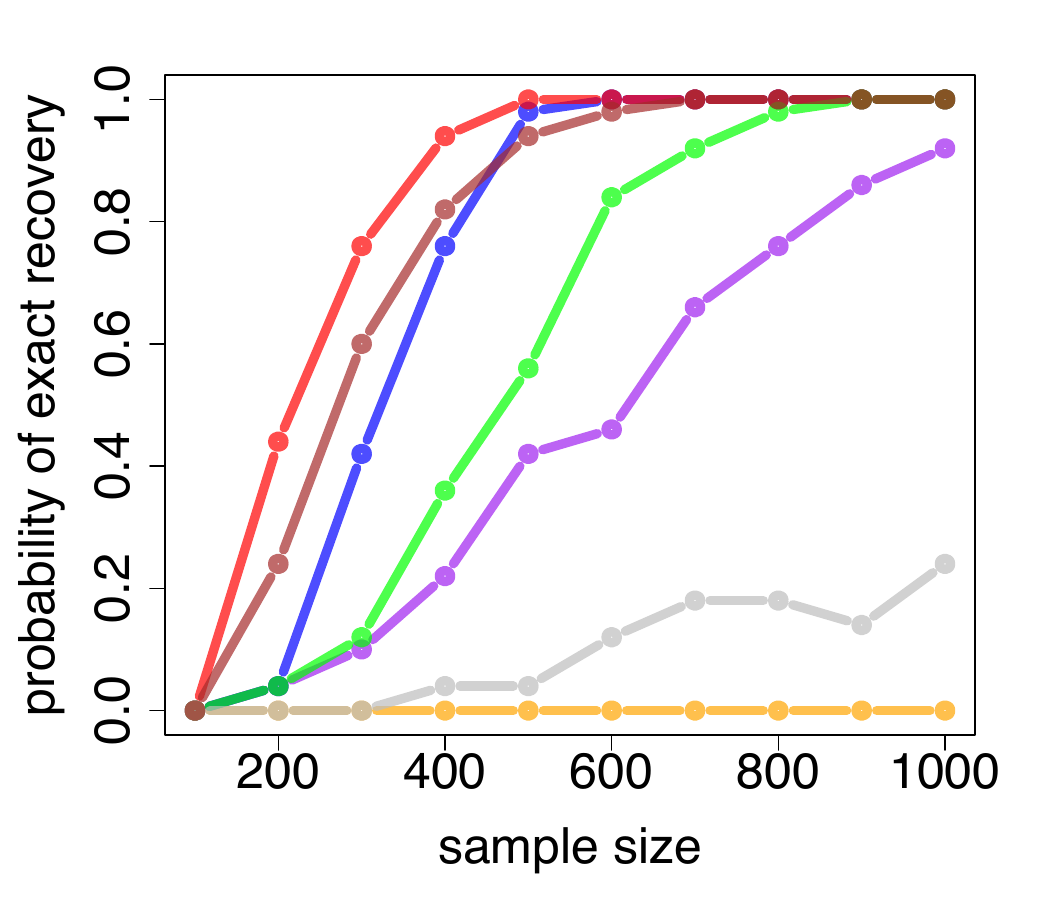}
	\caption{Model 2 (SNR $1.23$)}
	\label{fig:2}
	\end{subfigure}
	\begin{subfigure}[t]{0.3\textwidth}
		\centering
		\includegraphics[width=1\linewidth]{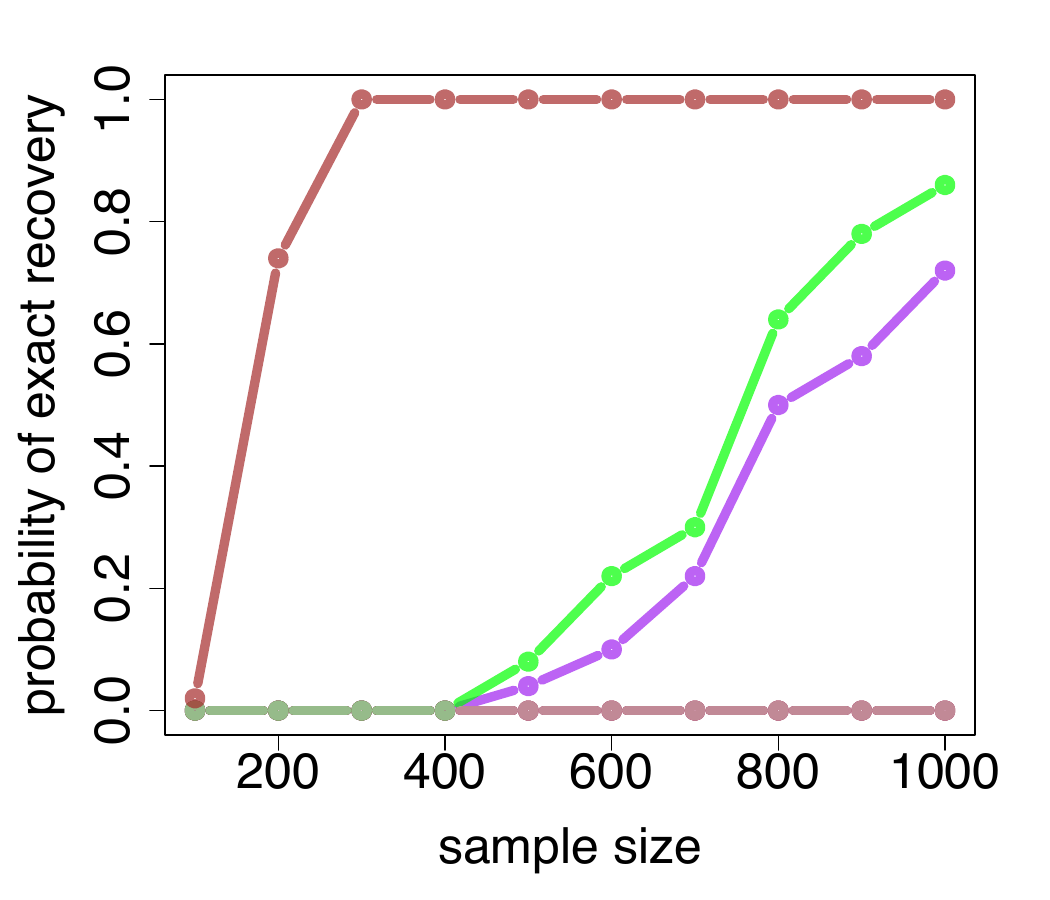}
		\caption{Model 3 (SNR $2.00$)}
				\label{fig:3}
	\end{subfigure}
	\begin{subfigure}[t]{0.3\textwidth}
		\centering
		\includegraphics[width=1\linewidth]{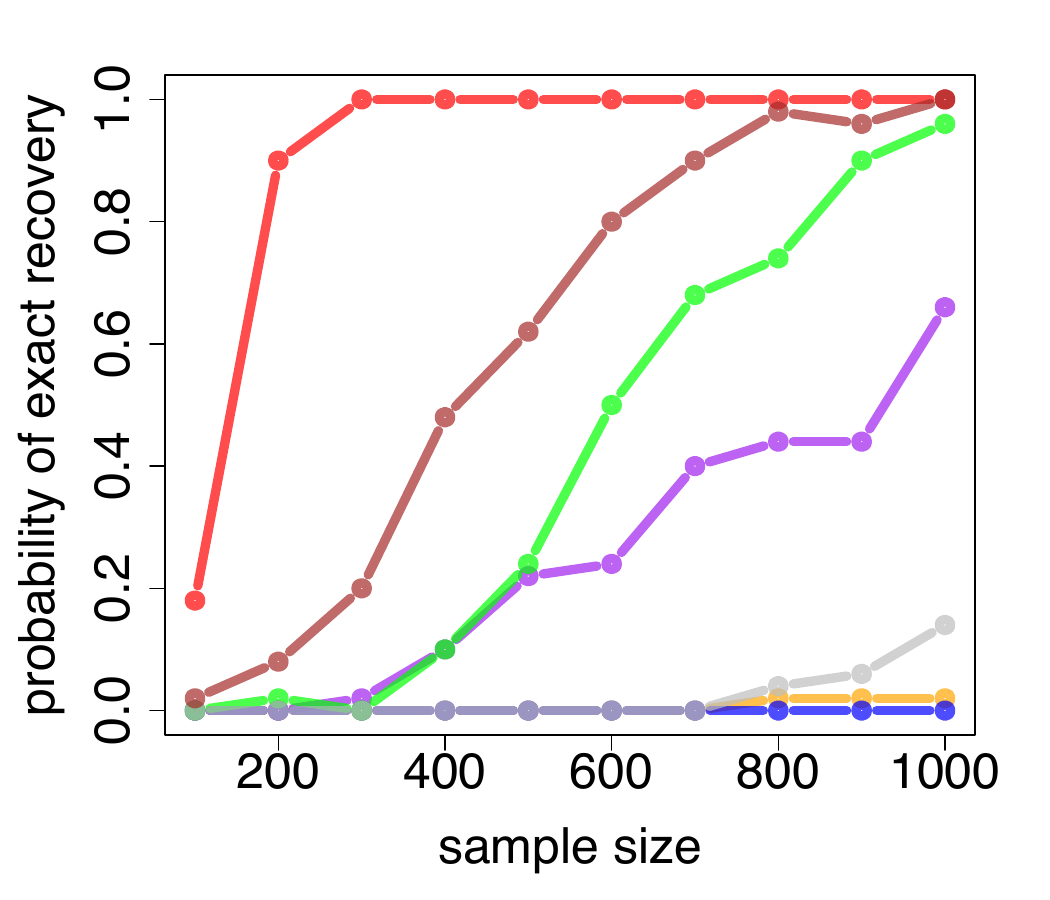}
		\caption{Model 4 (SNR $9.01$)}
		\label{fig:4}
	\end{subfigure}
	\begin{subfigure}[t]{0.3\textwidth}
	\centering
	\includegraphics[width=1\linewidth]{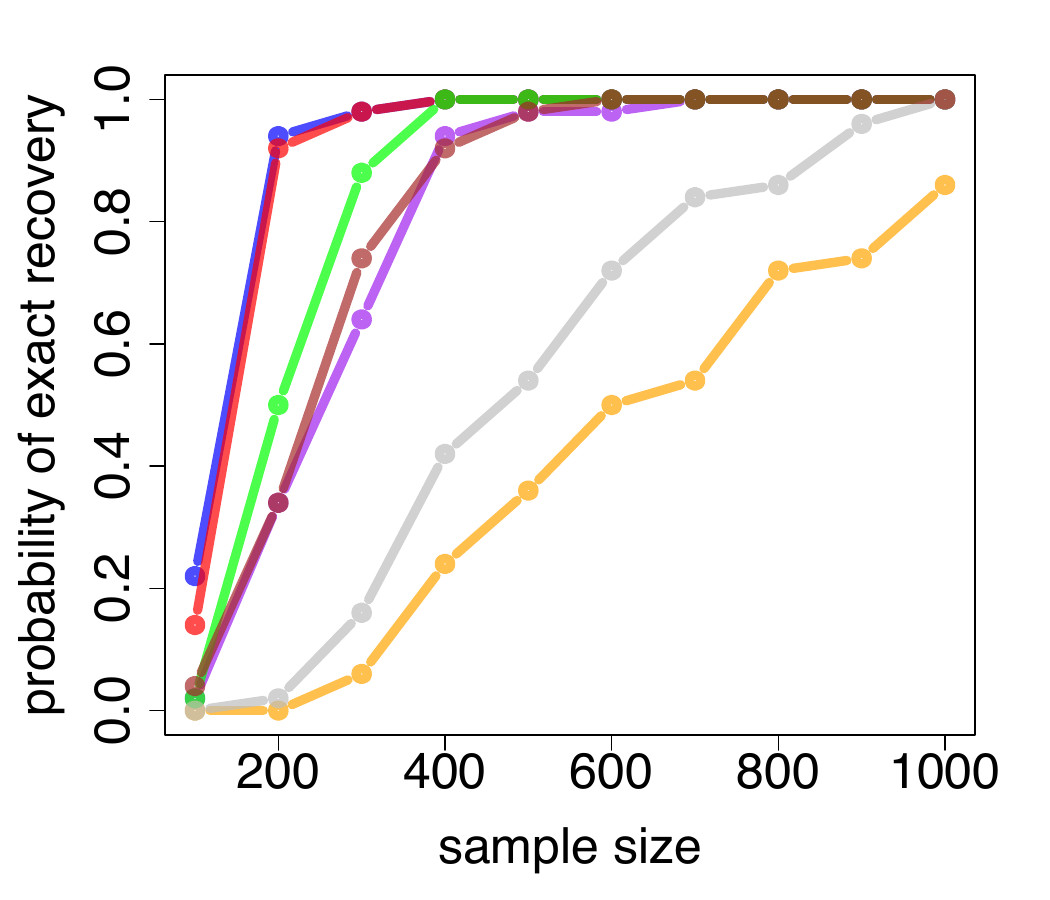}
	\caption{Model 5 (SNR $1.76$)}
	\label{fig:5}
	\end{subfigure}
	\begin{subfigure}[t]{0.3\textwidth}
	\centering
	\includegraphics[width=1\linewidth]{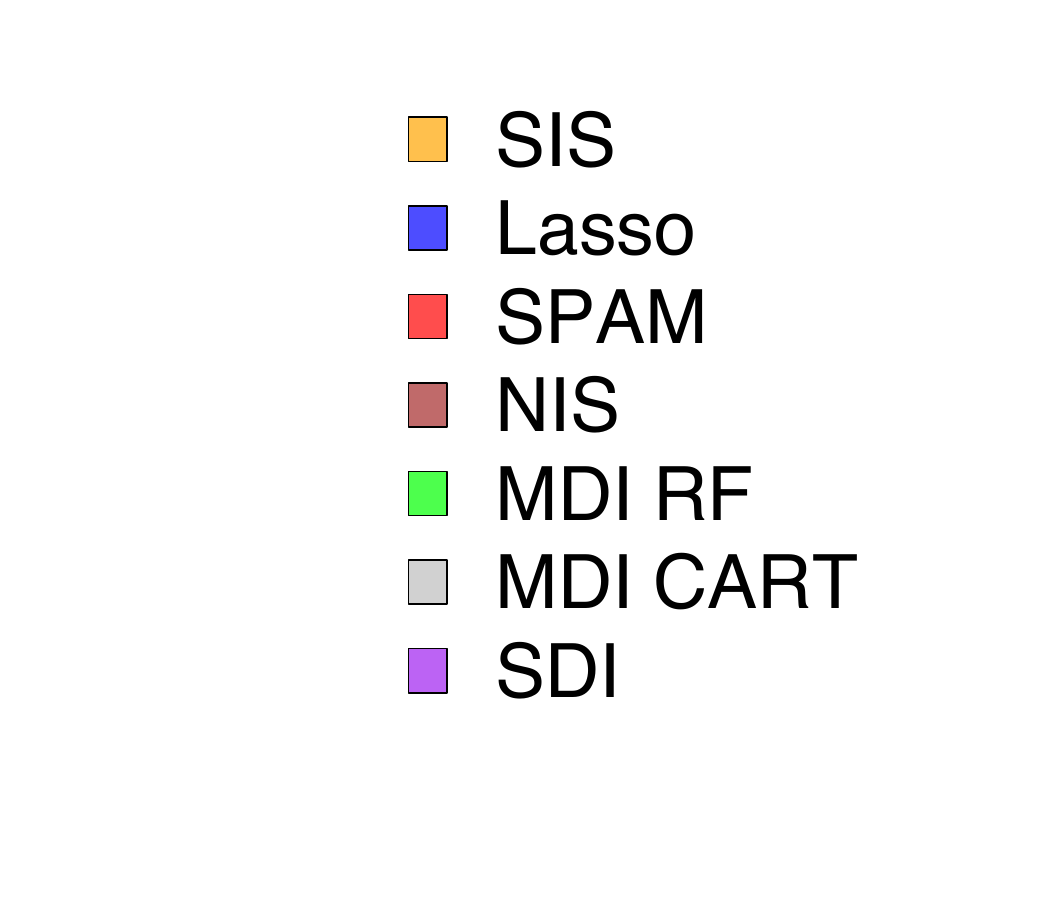}
	\end{subfigure}

	\caption{Plots of estimated $\BP(\widehat{\calS}=\calS)$ as $n$ increases for various models (approximate signal to noise ratio in parentheses).}
	\label{fig:exact}

	\end{figure} 




\subsection{Data driven choices of $\gamma_n$} \label{gamma_n_2}
We now consider the case when $s$ is unknown. As we have already demonstrated the performance of SDI in Section \ref{exact}, here we examine how well
the data-driven thresholding methods discussed in Section \ref{gamma_n} can estimate the true sparsity level $s$.

We first consider Model 5 and fix the sparsity level $s=4$ and the sample size $n=1000$.
In Figure \ref{fig:isdi}, we plot the probability of exact support recovery (averaged over $50$ independent replications) using the permutation method against the number of permutations. After a very small number of permutations, the significance level and performance of the algorithm appear to stabilize. 
For comparison, in Figure \ref{fig:elbow}, we show the graph of the ranked impurity reductions (SDI importance measures) used for the elbow method. When there is no correlation between irrelevant and relevant variables, as is the case with Model 5, the permutation method may be more precise than the elbow method.

However, as discussed in Section \ref{gamma_n}, the permutation method may be inaccurate if the irrelevant and relevant variables are correlated. To illustrate this, we now consider Model 1 with sparsity level $s=4$.
As we can see from Figure \ref{fig:isdi_cor}, the permutation method fails completely. Using $20$ permutations, it chooses a threshold $\gamma_n \approx 0.024$,
which will lead to all variables being selected, as can be seen from the ranked impurity reductions shown in Figure \ref{fig:elbow_half}. In other words, when there is strong correlation between the variables, the permutation method significantly underestimates the threshold $\gamma_n$, which in turn, creates too many false positives. In contrast, as can be seen from Figure \ref{fig:elbow_half}, the ranked impurity reductions still exhibit a distinct ``elbow'', from which the relevant and irrelevant variables can be discerned. 

\begin{figure}[H]
	\centering
\begin{subfigure}[t]{0.4\textwidth}
	\centering
	\includegraphics[width=1\linewidth]{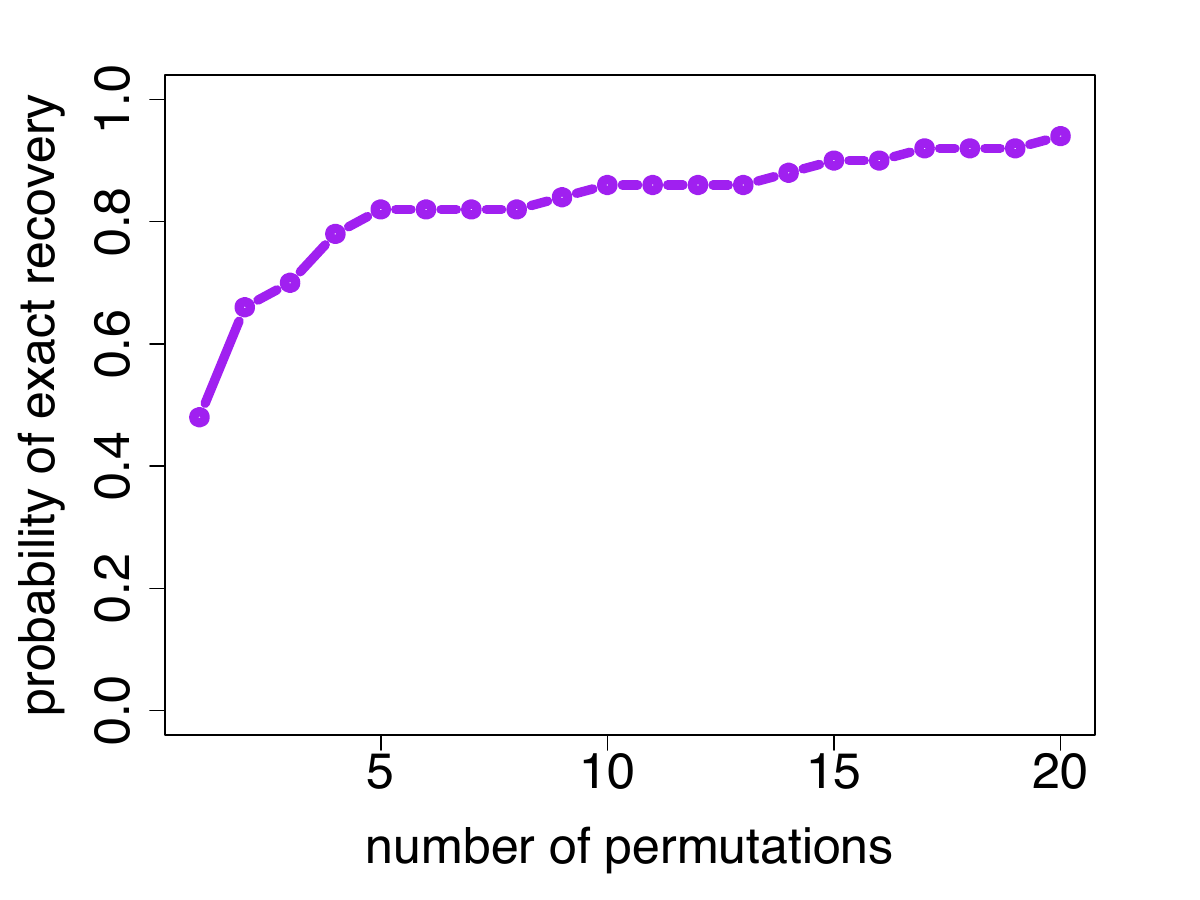}
	\caption{Permutation method (Model 5)}
	\label{fig:isdi}
\end{subfigure}
\hspace{0.3cm}
\begin{subfigure}[t]{0.4\textwidth}
	\centering
	\includegraphics[width=1\linewidth]{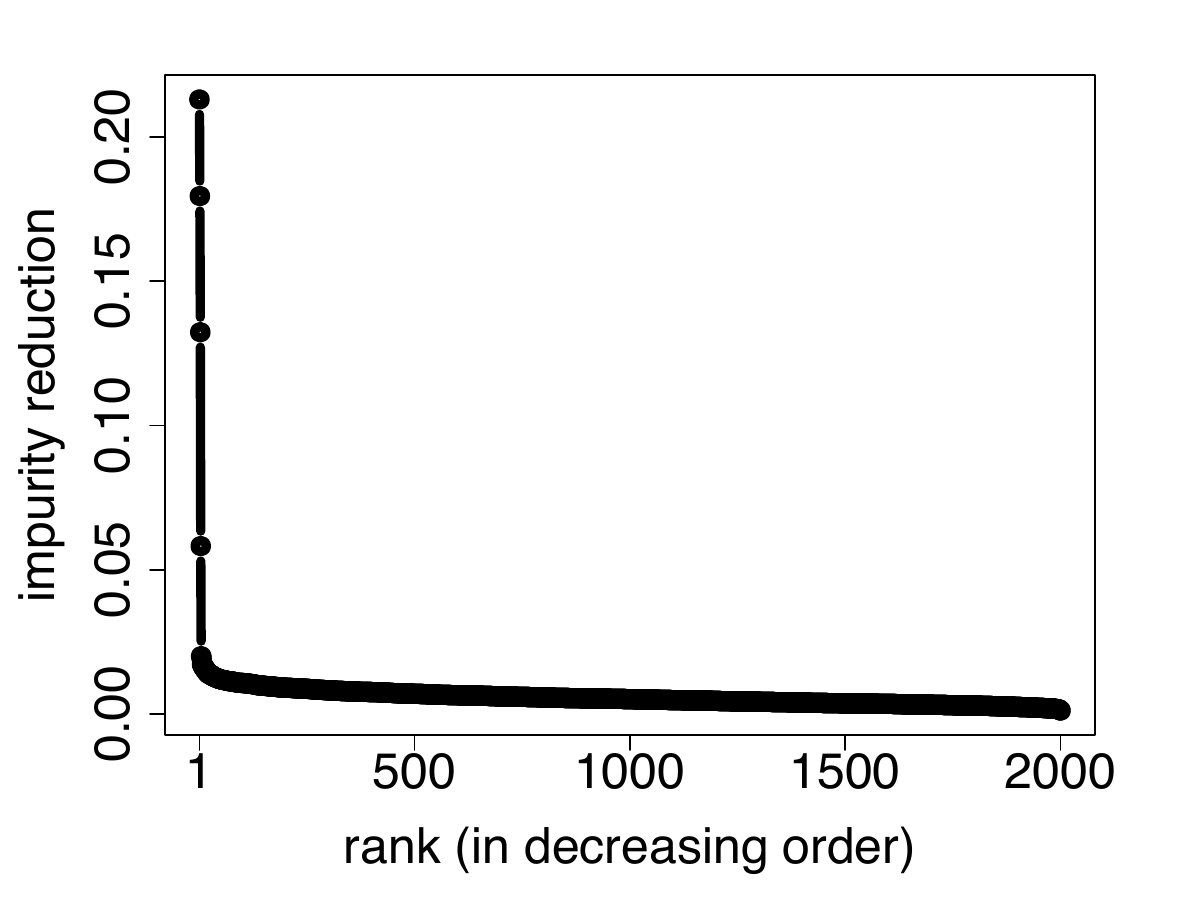}
	\caption{Elbow method (Model 5)}
	\label{fig:elbow}
\end{subfigure}
\hspace{0.3cm}
\begin{subfigure}[t]{0.4\textwidth}
	\centering
	\includegraphics[width=1\linewidth]{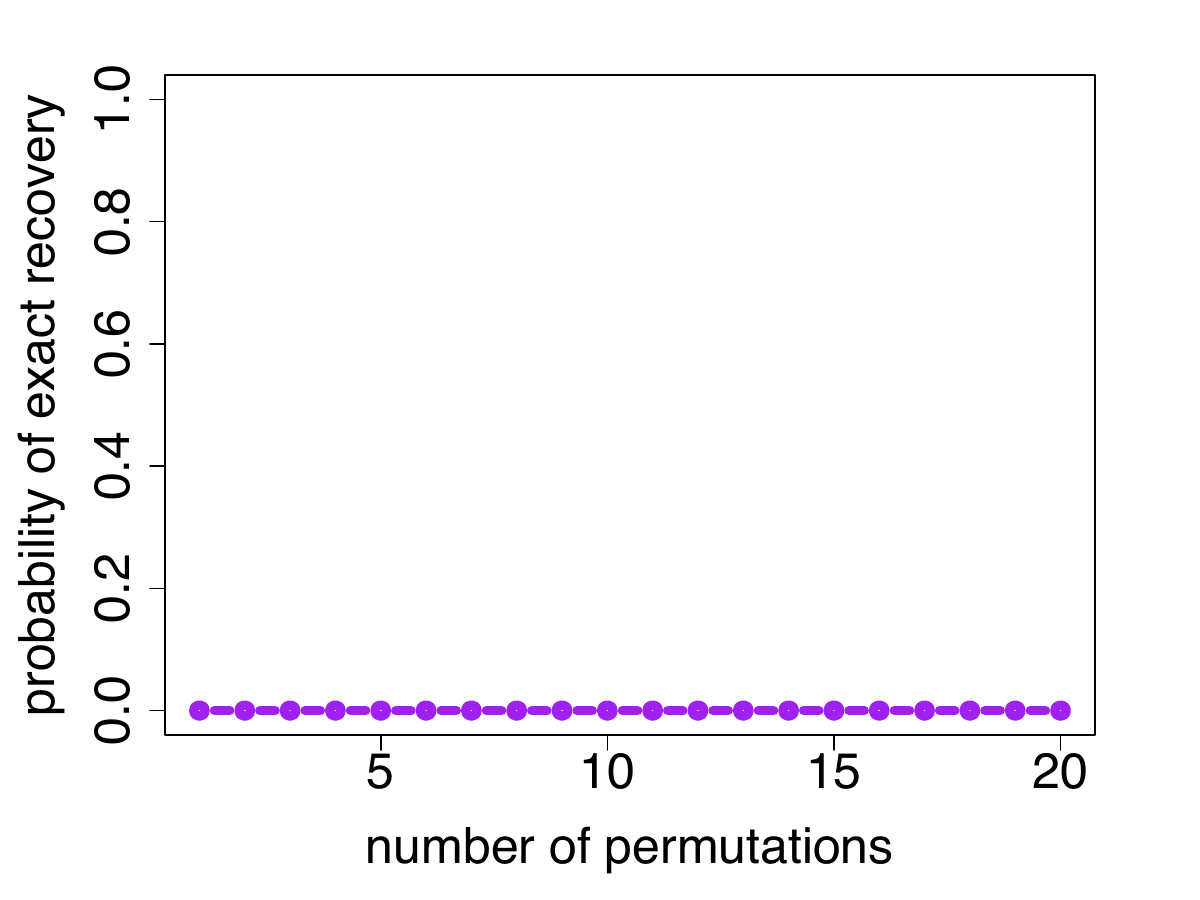}
	\caption{Permutation method (Model 1)}
	\label{fig:isdi_cor}
	\end{subfigure}
	\hspace{0.3cm}
	\begin{subfigure}[t]{0.4\textwidth}
	\centering
	\includegraphics[width=1\linewidth]{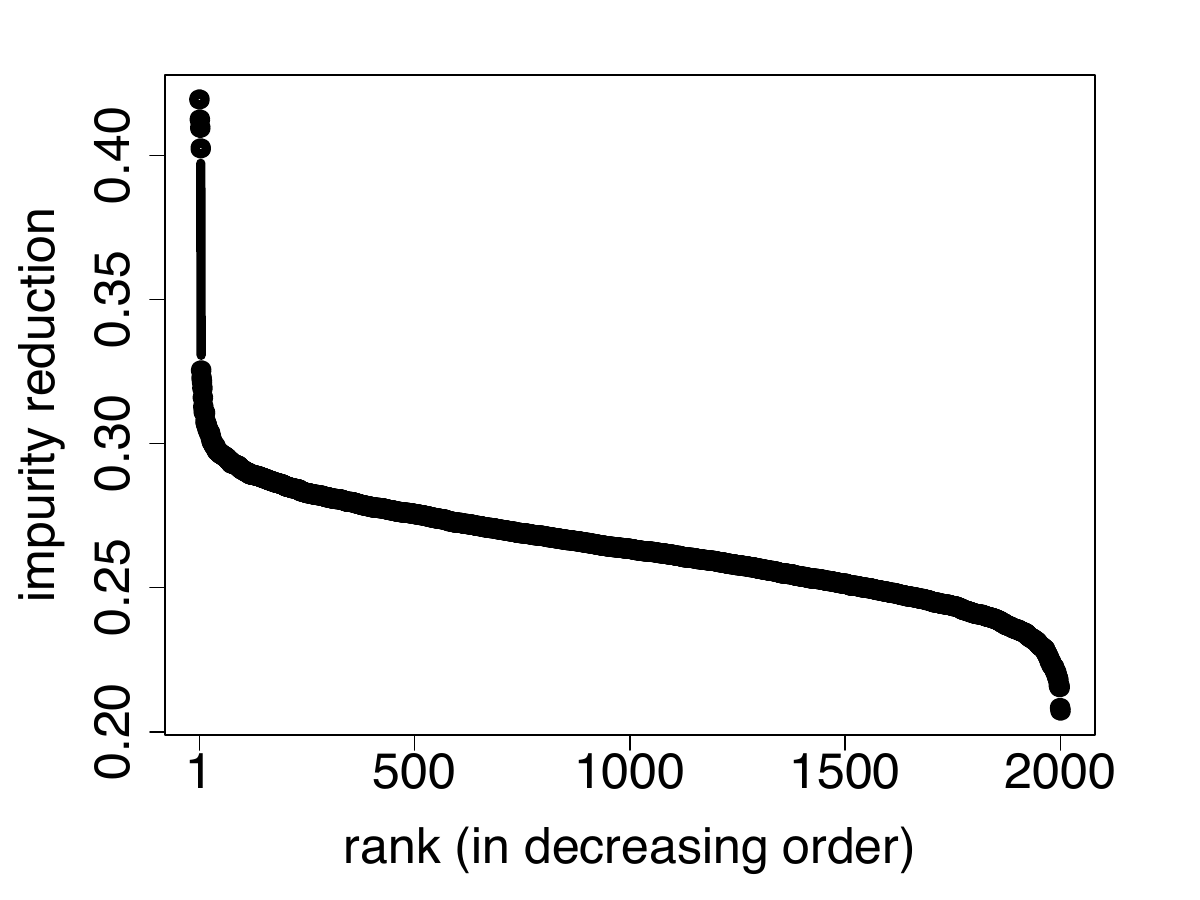}
	\caption{Elbow method (Model 1)}
	\label{fig:elbow_half}
\end{subfigure}
\caption{The probability of exact recovery by the number of random permutations used in the permutation method are shown in Figure \ref{fig:isdi} and Figure \ref{fig:isdi_cor} for models without and with correlation, respectively. Figures \ref{fig:elbow} and \ref{fig:elbow_half} show the plots of the corresponding ranked impurity reductions used for the elbow method.
}
\label{fig:gamma_n}
\end{figure}

\section{Discussion and conclusion} \label{conclusion}

In this paper, we developed a theoretically rigorous approach for variable selection based on decision trees. The underlying approach is simple, intuitive, and interpretable—we test whether a variable is relevant/irrelevant by fitting a decision stump to that variable and then determining how much it explains the variance of the response variable. 
Despite its simplicity, SDI performs favorably relative to its less interpretable competitors. Furthermore, due to the parsimony of the model, there is also no need to perform variable bandwidth selection or
calibrate the number terms in basis expansions.

On the other hand, we have sacrificed generality for analytical tractability. That is, decision stumps are poor at capturing interaction effects 
and, therefore, an importance measure built from a multi-level decision tree (such as MDI) may be more appropriate for models with more than just main effects. 
However, the presence of multi-level splits adds an additional layer of complexity to the analysis that, at the moment, we do not know how to overcome. It is also unclear how to leverage these additional splits to strengthen the theory. While out of the scope of the present paper, we view this as an important problem for future investigation. It is our hope that the tools in this paper can be used by other scholars to address this important issue. 

To conclude, we hope that our analysis of single-level decision trees for variable selection will shed further light on the unique benefits of tree structured learning.







\section{Appendix}
\setcounter{section}{0}
\def\thesection{\Alph{section}}




\section{Appendix}

\setcounter{section}{0}
\setcounter{subsection}{0}
\setcounter{subsubsection}{0}
\setcounter{equation}{0}
\setcounter{lemma}{0}
\setcounter{theorem}{0}
\setcounter{remark}{0}
\setcounter{proposition}{0}
\def\thesection{\Alph{section}}


\renewcommand{\theequation}{\thesection.\arabic{equation}}
\renewcommand{\thelemma}{\thesection.\arabic{lemma}}
\renewcommand{\thetheorem}{\thesection.\arabic{theorem}}
\renewcommand{\theremark}{\thesection.\arabic{remark}}
\renewcommand{\theproposition}{\thesection.\arabic{proposition}}

In this appendix, we first prove Lemma \ref{delta-tail} in detail in Appendix \ref{new-delta-tail}.
We then prove Theorem \ref{sis} on linear models with Gaussian variates in Appendix \ref{correlated predictors} and then use Theorem \ref{sis} to determine a sufficient sample size for model selection consistency (mentioned in Section \ref{comparisons}) in Appendix \ref{remark}. 
Finally, we prove Propositions \ref{proposition lipschitz} and \ref{proposition monotone} in Appendix \ref{23} and then prove Theorems \ref{theorem lipschitz ii} and \ref{theorem monotone ii} in Appendix \ref{45}. 

\section{Proof of Lemma \ref{delta-tail}} \label{new-delta-tail}
	Let $\pi$ be a permutation of the data such that $X_{\pi(1)j} \le X_{\pi(2)j} \le \cdots \le X_{\pi(n)j}$. 
	Recall from the representation \eqref{LR} that we have
	$$
	\widehat \Delta(X_j,Y) =   \max_{1 \le k \le n} \frac{k}{n} \Big( 1 - \frac{k}{n} \Big) \bigg(\underbrace{\frac{1}{k} \sum_{X_{ij} \le X_{\pi(k)j}} Y_i  - \frac{1}{n - k} \sum_{X_{ij} > X_{\pi(k)j}} Y_i}_{\text{(III)}}\bigg)^2.
	$$

	Now, since
	$$
	\sum_{i=1}^n\bigg(\frac{\mathbf{1}(X_{ij} \le X_{\pi(k)j})}{k} - \frac{\mathbf{1}(X_{ij} > X_{\pi(k)j})}{n-k}\bigg) = 0,
	$$
	we can rewrite (III) as
	\begin{align*}
	&\underbrace{\frac{1}{k} \sum_{X_{ij} \le X_{\pi(k)j}} (Y_i - f_j(X_{ij}))}_{\text{(a)}} 
	- \underbrace{\frac{1}{n - k} \sum_{X_{ij} > X_{\pi(k)j}} (Y_i - f_j(X_{ij}))}_{\text{(b)}} \\ 
	&+ \underbrace{\sum_{i=1}^n \Bigg(f_j(X_{ij}) - \frac{1}{n} \sum_{i=1}^n f_j(X_{ij}) \Bigg)\bigg(\frac{\mathbf{1}(X_{ij} \le X_{\pi(k)j})}{k} - \frac{\mathbf{1}(X_{ij} > X_{\pi(k)j})}{n-k}\bigg)}_{\text{(c)}}.
	\end{align*} 
	Therefore, we have that
	\begin{equation}\label{sos}
	\begin{aligned}
	\widehat \Delta(X_j,Y) & = \max_{1\le k \le n}  \frac{k}{n} \Big( 1 - \frac{k}{n} \Big) (\text{(a)} - \text{(b)} + \text{(c)})^2 \\
	& \le 3\max_{1\le k \le n} \frac{k}{n} \Big( 1 - \frac{k}{n} \Big)\text{(a)}^2 + 3\max_{1\le k \le n}  \frac{k}{n} \Big( 1 - \frac{k}{n} \Big)\text{(b)}^2 + \\ & \qquad 3 \max_{1\le k \le n} \frac{k}{n} \Big( 1 - \frac{k}{n} \Big)\text{(c)}^2,
	\end{aligned}
	\end{equation}
	where we use, in succession, the inequality $ (x-y+z)^2 \leq 3(x^2+y^2+z^2) $ for any real numbers $ x $, $ y $, and $ z $, and the fact that the maximum of a sum is at most the sum of the maxima.
	To finish the proof, we will bound the terms involving $\text{(a)}^2$, $\text{(b)}^2$, and $\text{(c)}^2$ separately. 
	
	For the last term in \eqref{sos}, notice that by the Cauchy-Schwartz inequality we have
	\begin{align*}
	\frac{k}{n} \Big(1 - \frac{k}{n} \Big)\text{(c)}^2 &=\frac{k}{n} \Big(1 - \frac{k}{n} \Big) 
	\Bigg[\sum_{i=1}^n \Big(f_j(X_{ij}) - \frac{1}{n} \sum_{i=1}^n f_j(X_{ij}) \Big) \times\\ 
	& \qquad\qquad\qquad \bigg(\frac{\mathbf{1}(X_{ij} \le X_{\pi(k)j})}{k} - \frac{\mathbf{1}(X_{ij} > X_{\pi(k)j})}{n-k}\bigg)\Bigg]^2 \\
	&\le \frac{k}{n} \Big(1 - \frac{k}{n} \Big) \sum_{i=1}^n\Big( f_j(X_{ij}) - \frac{1}{n} \sum_{i=1}^n f_j(X_{ij}) \Big)^2  \times \\ 
	& \qquad\qquad\qquad \sum_{i=1}^n \bigg(\frac{\mathbf{1}(X_{ij} \le X_{\pi(k)j})}{k} - \frac{\mathbf{1}(X_{ij} > X_{\pi(k)j})}{n-k}\bigg)^2,
	\end{align*}
	which is exactly equal to
	$$
	\frac{k}{n} \Big(1 - \frac{k}{n} \Big) 
	\Big[n \widehat{\Var}(f_j(X_j))  \Big(k \cdot \frac{1}{k^2} + (n-k) \cdot \frac{1}{(n-k)^2} \Big)\Big] = \widehat{\Var}(f_j(X_j)).
	$$
	Therefore we have shown that
	\begin{equation} \label{(c)}
	\max_{1\le k \le n} \frac{k}{n} \Big(1 - \frac{k}{n} \Big)\text{(c)}^2\le \widehat{\Var}(f_j(X_j)).
	\end{equation}
	To bound the first term in \eqref{sos}, by a union bound we have that 
	\begin{align}
	&\BP\Big(\max_{1\le k \le n}  \frac{k}{n} \Big(1 -\frac{k}{n}\Big) \text{(a)}^2 > \frac{\xi^2}{6}\Big) \nonumber  \\
	&\le \sum_{k=1}^n \BP\Big(\frac{k}{n} \Big(1 -\frac{k}{n}\Big)\text{(a)}^2  > \frac{\xi^2}{6} \Big) \nonumber \\
	&= \sum_{k=1}^n \BP\Big(\frac{k}{n} \Big(1 -\frac{k}{n}\Big) \Big(\frac{1}{k} \sum_{X_{ij} \le X_{\pi(k)j}} (Y_i - f_j(X_{ij})) \Big)^2 > \frac{\xi^2}{6} \Big). \label{probtail}
	\end{align}
	Next, notice that, conditional on $ X_{1j}, \dots, X_{nj} $, $ \sum_{X_{ij} \le X_{\pi(k)j}} (Y_i - f_j(X_{ij})) $ is a sum of $ k $ independent, sub-Gaussian, mean zero random variables.
	Thus, by the law of total probability, we have that \eqref{probtail} is equal to
	\begin{equation*}
	\sum_{k=1}^n \BE\Bigg[\BP\Bigg(\bigg|\frac{1}{k} \sum_{X_{ij} \le X_{\pi(k)j}} (Y_i - f_j(X_{ij})) \bigg| > \xi \sqrt{\frac{n^2}{6k(n-k)}} \quad \Bigg|  \quad X_{1j}, \dots, X_{nj} \Bigg)\Bigg] 
	\end{equation*}
	and, by Hoeffding's inequality for sub-Gaussian random variables, is bounded by
	\begin{equation*}
	\sum_{k=1}^n 2 \exp\Big(-k \frac{\xi^2n^2}{12k(n-k) \sigma^2_{Z_j}} \Big) \le 2n \exp\Big(-\frac{\xi^2n}{12\sigma^2_{Z_j}}\Big).
	\end{equation*}
Note that here we have implicitly used the fact that $ \BE[\exp(\lambda Z_j) | X_j] \leq \exp(\lambda^2\sigma^2_{Z_j}/2) $.
	It thus follows that with probability at least $1 - 2n \exp\big(-\frac{\xi^2n}{12\sigma^2_{Z_j}}\big)$ that
	\begin{equation} \label{(a)}
	\max_{1\le k \le n} \frac{k}{n} \Big(1 -\frac{k}{n}\Big) \text{(a)}^2 \le \frac{\xi^2}{6}.
	\end{equation} 
	A similar argument shows that with
	probability at least $1 - 2n \exp\Big(-\frac{\xi^2n}{12\sigma^2_{Z_j}}\Big)$, the second terms in \eqref{sos} obeys
	\begin{equation}\label{(b)}
	\max_{1\le k \le n} \frac{k}{n} \Big(1 -\frac{k}{n}\Big) \text{(b)}^2 \le \frac{\xi^2}{6}.
	\end{equation} 
	Therefore, substituting \eqref{(c)}, \eqref{(a)}, and \eqref{(b)} into \eqref{sos} and using a union bound, it follows that with probability at least $ 1- 4n \exp\big(-\frac{\xi^2n}{12\sigma^2_{Z_j}}\big)$, 
	$$
	\widehat \Delta(X_j,Y) \leq 3\widehat\Var(f_j(X_j)) + \xi^2.
	$$
	


\section{Proof of Theorem \ref{sis}} \label{correlated predictors}

The goal of this section is to prove Theorem \ref{sis}. In the first section, we prove the lower bound \eqref{sis-lower} and in the second section, we prove the upper bound \eqref{sis-upper}.

Throughout this section, for brevity, we let $ \rho_j = {\rm{Cor}}(X_j, Y) \neq 0 $. 

\subsection{Proof of the lower bound \eqref{sis-lower}}
Choosing $h(X_j) = X_j$ (which is monotone) in Lemma \ref{corr} to get that
$$
\widehat\Delta(X_j, Y) \geq \frac{1}{\log(2n)+1} \times\widehat{\Cov}^2\Bigg(\frac{X_j}{\sqrt{\widehat{\Var}(X_j)}}, Y\Bigg) = \frac{\widehat \Var(Y)}{\log(2n)+1} \times\widehat{\rho}^{\;2}(X_j, Y).
$$
Now observe that $ \widehat\rho\, ( X_j, Y) $ is the empirical Pearson sample correlation between two correlated normal distributions.
If $\rho_j> 0$, by \citep[Equation (44)]{hotelling1953correlation}, we have that
\begin{equation} \label{+}
\begin{aligned}
&\BP(\widehat \rho\,(X_j,Y) > (1 - \delta)\rho_j )\\
&= 1 - \BP(\widehat \rho\,(-X_j,Y) > - (1-\delta)\rho_j) \\
&\sim 1 - (2\pi)^{-1/2} \frac{\Gamma(n)}{\Gamma(n+1/2)}(1-\rho_j^2)^{n/2} (1-[(1-\delta)\rho_j]^2)^{(n-1)/2} \\
&\qquad \times (-(1-\delta)\rho_j - (-\rho_j))^{-1} (1 - (-\rho_j)(-(1-\delta)\rho_j) )^{-n + 3/2}(1+\calO(n^{-1})).
\end{aligned}
\end{equation}
If $\rho_j <0$ we can show the same bound on $\BP(\widehat \rho\,(X_j,Y) < (1 - \delta)\rho_j)$. Again by \citep[Equation (44)]{hotelling1953correlation}, we have the similar bound
\begin{equation} \label{-}
\begin{aligned}
&\BP(\widehat \rho\,(X_j,Y) < (1 - \delta)\rho_j ) \\
&= 1 - \BP(\widehat \rho\,(X_j,Y) > (1-\delta)\rho_j) \\
&\sim 1 - (2\pi)^{-1/2} \frac{\Gamma(n)}{\Gamma(n+1/2)}(1-\rho_j^2)^{n/2} (1-[(1-\delta)\rho_j]^2)^{(n-1)/2} \\
& \qquad \times ((1-\delta)\rho_j - \rho_j)^{-1} (1 - \rho_j\times (1-\delta)\rho_j) )^{-n + 3/2}(1+\calO(n^{-1})).
\end{aligned}
\end{equation}
Therefore because of \eqref{+} and \eqref{-}, regardless of the sign of $\rho_j$, it follows that there exists a universal constant $C_0$
for which 
\begin{align}
&\BP(|\widehat \rho\,(X_j,Y)| > (1 - \delta)|\rho_j|) \nonumber\\
&\ge 1- \frac{C_0}{\sqrt{2\pi} \delta |\rho_j|} \frac{\Gamma(n)}{\Gamma(n+1/2)} (1-\rho_j^2)^{\frac{n}{2}} (1-(1-\delta)^2\rho_j^2)^{\frac{n-1}{2}} (1 -(1-\delta)\rho_j^2)^{-n + \frac{3}{2}} \nonumber \\
&\ge 1-  \frac{C_0}{\sqrt{n} \delta |\rho_j|} \exp\big(-\rho_j^2n/2 - (1-\delta)^2\rho_j^2(n-1)/2 +(1-\delta)\rho_j^2(n-3/2)\big) \label{second} \\
&= 1- \frac{C_0}{\sqrt{n \delta^2 \rho_j^2}} \exp\big(-\rho_j^2n\delta^2/2+\rho_j^2(1-\delta)^2/2 -3(1-\delta)\rho_j^2/2\big) \nonumber \\
&\ge 1- \frac{C_0}{\sqrt{n \delta^2 \rho_j^2}} \exp(-\rho_j^2n\delta^2/2), \nonumber
\end{align}
where we used $\exp(x) \ge 1 + x$ and Wendel's inequality \citep{wendel2001gamma} $\frac{\Gamma(n)}{\Gamma(n+1/2)} \leq \sqrt{\frac{n+1/2}{n}} \frac{1}{\sqrt{n}} \le \sqrt{\frac{2\pi}{n}}$
in the second inequality \eqref{second}. 
Thus, we have that with probability at least $ 1 - \frac{C_0}{\sqrt{n \delta^2 \rho_j^2}} \exp(-\rho_j^2n\delta^2/2)$ that
$$
\widehat\Delta(X_j, Y) \geq \frac{(1-\delta)^2\widehat{\Var}(Y)\rho^2_j}{
	\log(2n)+1} \iff \widehat\rho^{\,2}(\widehat Y(X_j),Y) \ge \frac{(1-\delta)^2\rho_j^2}{\log(2n)+1}.
$$
This completes the first half of the proof of Theorem \ref{sis}. 
	
	\subsection{Proof of the upper bound \eqref{sis-upper}}

We first state the following sample variance concentration inequality, which will be helpful.

\begin{lemma}\label{chi-squared}
	Let $Z_1, \dots, Z_n$ be i.i.d. $ \calN(0, \sigma^2_Z) $. For any $0<\delta<1$, we have
	\begin{equation}\label{lower}
	\BP(\widehat{\Var}(Z)  \ge (1-\delta) \frac{n-1}{n}\sigma^2_Z) \ge 1-\exp(-\delta^2(n-1)/4)
	\end{equation}
	and
	\begin{equation}\label{upper}
	\BP (\widehat{\Var}(Z)  \le (1+\delta) \frac{n-1}{n}\sigma^2_Z) \ge1-\exp(-(n-1)(1+\delta-\sqrt{1+2\delta})/2).
	\end{equation}
\end{lemma}
\begin{proof}[Proof of Lemma \ref{chi-squared}] 
	Since $Z_i$ are independent and normally distributed, by Cochran's theorem we have $ \widehat{\Var}(Z) \sim \frac{\sigma^2_Z}{n}\chi^2_{n-1} $. In the notation of \citep{laurent2000adaptive}, choosing $D = n-1$ and $x = \delta^2(n-1)/4$ for the chi-squared concentration inequality (4.4) in \citep{laurent2000adaptive}, we have that
	\begin{align*}
	\mathbb{P}\Big( \widehat{\Var}(Z)  \ge (1-\delta) \frac{n-1}{n}\sigma^2_Z\Big) 
	&= 1 - \BP(\chi_{n-1}^2 < (1-\delta)(n-1))\\
	&\geq 1-\exp(-\delta^2(n-1)/4),
	\end{align*}
	proving \eqref{lower}.
	For \eqref{upper}, 
	choosing $D=n-1$ and $x = (n-1)(1+\delta-\sqrt{1+2\delta})/2$
	in \citep[Equation (4.3)]{laurent2000adaptive} we see that
	\begin{align*}
	\mathbb{P}\Big( \widehat{\Var}(Z)  \le (1+\delta) \frac{n-1}{n}\sigma^2_Z\Big) &= 1 - \BP(\chi_{n-1}^2 >(1+\delta)(n-1)) \\
	&\geq 1-\exp(-(n-1)(1+\delta-\sqrt{1+2\delta})/2). \qedhere
	\end{align*}
\end{proof}
Now we are ready to prove the upper bound \eqref{sis-upper}. 
We begin with the inequality \eqref{Delta-additive}, as shown in the proof sketch of Theorem \ref{sis}. We aim to upper bound the right hand side of  \eqref{Delta-additive} using Lemma \ref{chi-squared}.
Since the samples $ X_{1j}, \dots, X_{nj} $ are i.i.d., using \eqref{upper} and choosing $\delta = 1$, we find that with probability at least $1-\exp\big(-(n-1)\frac{2-\sqrt{3}}{2}\big) \ge 1 - \exp(-(n-1)/16) $, we have that  $\widehat{\Var}(X_j)  \le 2\sigma^2_{X_j}$. Similarly, choosing $ \delta = 1/2$ in \eqref{lower}, we also have that with probability at least $1 - \exp(-(n-1)/16)$ that $\widehat \Var(Y) \ge \sigma_Y^2/4$. Substituting these concentration inequalities into the right hand side of \eqref{Delta-additive}, it follows by a union bound
that with probability at least $1 - 4n\exp(-n\delta^2/12) -2 \exp(-(n-1)/16) $,
$$
\widehat \Delta(X_j, Y)  \le 24 \widehat \Var(Y) \rho_j^2 + 4\delta^2\widehat \Var(Y) \iff \widehat\rho^{\,2}(\widehat Y(X_j),Y) \le 24 \rho_j^2 + 4\delta^2.
$$
Finally, noticing that $\sqrt{24 \rho_j^2 + 4\delta^2} < 5 |\rho_j| + 2 \delta$ completes the proof.

 \section{Proof of Model Selection Consistency for Linear Models} \label{remark}

Recall the setting mentioned under the heading ``\textbf{\textit{Minimum sample size for consistency}}''  in Section \ref{comparisons}, which considers the same linear model with Gaussian variates from Theorem \ref{sis}. To reiterate, we assume that $ \bSigma = \mathbf{I}_{p\times p} $
is the $ p\times p $ identity matrix, $ \sum_{k=1}^p \beta^2_k = \calO(1)$, and $\min_{j\in\calS}|\beta_j|^2 \asymp 1/s$, all of which are special cases of the more general setting considered in \citep[Corollary 1]{wainwright2009information}.
Under these assumptions, we then have $ \rho^2(X_j, Y) = \beta^2_j/(\sigma^2+\sum_{k=1}^p \beta^2_k) \gtrsim 1/s$ for any $j\in\calS$ and $ \rho(X_j, Y) = 0 $ for $ j \in \calS^c $. Our goal is to show that $n \asymp s \log (n) \log(n(p-s))$ samples suffice for high probability model selection consistency.


Choosing $\delta=1/2$ in \eqref{sis-lower} applied to $j\in\calS$ and using \eqref{impurity corr}, there exists a universal positive constant $C_0$ such that with probability at least $ 1 - \frac{2C_0}{\sqrt{n \rho^2(X_j,Y)}} \exp(-n\rho^2(X_j,Y)/8)$, we have
\begin{align*}
\widehat\Delta(X_j,Y) &\ge \widehat{\Var}(Y) \times \frac{\rho^2(X_j,Y)}{4(\log(2n)+1)} \\
& = \frac{\widehat{\Var}(Y) }{4(\log(2n)+1)} \times \frac{\beta_j^2}{\sigma^2+\sum_{k=1}^p \beta_k^2} \\
& \gtrsim \frac{\widehat{\Var}(Y)}{s\log(n)}.
\end{align*}
Therefore by a union bound over all $s$ relevant variables, we have that with probability at least $ 1 -s \max_{j\in\calS}\{\frac{2C_0}{\sqrt{n  \rho^2(X_j,Y)}} \exp(-n\rho^2(X_j,Y)/8)\}$, 
\begin{equation} \label{sis-lower2}
\widehat\Delta(X_j,Y) \gtrsim \frac{\widehat{\Var}(Y)}{s\log(n)} \qquad \forall j \in \calS.
\end{equation}
Furthermore, by applying \eqref{sis-upper} for $j\in \calS^c$ (and noting that $\rho(X_j,Y)=0$) and using \eqref{impurity corr}, 
with probability at least $ 1 - 4n \exp(-\delta^2 n/12) - 2\exp(-(n-1)/16) $, we have
$$
\widehat\Delta(X_j,Y) \le 4\widehat{\Var}(Y) \delta^2.
$$
Therefore by a union bound over all $p-s$ irrelevant variables we have that with probability at least $ 1 - 4n(p-s) \exp(-\delta^2 n/12) - 2(p-s)\exp(-(n-1)/16)$,
$$
\widehat\Delta(X_j,Y) \le 4\widehat{\Var}(Y) \delta^2 \qquad \forall j \in \calS^c.
$$
Now, choosing $\delta^2 = \frac{C_3}{s\log(n)}$ for some appropriate constant $C_3>0$ which only depends on $\sigma^2$ to match \eqref{sis-lower2}, we see by a union bound that
\begin{equation} \label{consistency}
\begin{aligned}
\BP(\widehat{\calS} = \calS) &\geq  1 - \max_{j\in\calS}\Big\{\frac{2C_0s}{\sqrt{n \rho^2(X_j,Y)}} \exp\Big(-\frac{n\rho^2(X_j,Y)}{8}\Big)\Big\} \\
& \qquad -  4n(p-s) \exp\Big(-\frac{C_3 n}{12s\log(n)}\Big) - 2(p-s)\exp\Big(-\frac{(n-1)}{16}\Big).
\end{aligned}
\end{equation}
Since $\rho^2(X_j, Y) \gtrsim 1/s$ for all $j\in\calS$, \eqref{consistency} implies that if $ n(p-s) \exp\big(-\frac{C_3 n}{12s\log(n)}\big) \rightarrow 0 $, then $  \BP(\widehat{\calS} = \calS) \rightarrow 1 $. Hence, a sufficient sample size for consistent support recovery is
$$
n \asymp s \log (n) \log(n(p-s)),
$$
as desired.

\section{Proof of Propositions \ref{proposition lipschitz} and \ref{proposition monotone}} \label{23}

This section will mainly be devoted to proving Proposition \ref{proposition lipschitz}. First we will present the machinery which will be used to prove Proposition \ref{proposition lipschitz}. At the end of the section we will complete the proof of Proposition \ref{proposition lipschitz} and prove Proposition \ref{proposition monotone} by recycling and simplifying the proof of Proposition \ref{proposition lipschitz}.



First, we state and prove a lemma that will be used in later proofs. Though stated in terms of general probability measures, we will be specifically interested in the case where $ \mathbb{P}  $ is the empirical probability measure $ \mathbb{P}_n $ and $ \mathbb{E} $ is the empirical expectation $ \mathbb{E}_n $, both with respect to a sample of size $n$.


\begin{lemma} \label{measure}
	For any random variables $ U $ and $ V $ with finite second moments with respect to a probability measure $\BP $,
	$$
	\Cov_\BP\Big(\frac{U}{\sqrt{\Var_\BP(U)}}, V\Big) \geq \sqrt{\Var_\BP(V)} - 2\sqrt{\mathbb{E}_\BP[(U-V)^2]}.
	$$
\end{lemma}
\begin{proof}[Proof of Lemma \ref{measure}]
	First note that by the triangle inequality, 
	\begin{equation} \label{varcov1}
	\big|\sqrt{\Var_\BP(U)}-\sqrt{\Var_\BP(V)}\big| \leq \sqrt{\Var_\BP(U-V)} \leq \sqrt{\Expect_\BP{[(U-V)^2]}}.
	\end{equation}
	To complete the proof, we write
	\begin{equation} \label{varcov3}
	\Cov_\BP(U, V) = 
	\sqrt{\Var_\BP(U)}\sqrt{\Var_\BP(V)} + \big(\Cov_\BP(U, V) -\sqrt{\Var_\BP(U)}\sqrt{\Var_\BP(V)}\big)
	\end{equation}
	and apply \eqref{varcov1} to arrive at
	\begin{align}
	&\big|\Cov_\BP(U, V) -\sqrt{\Var_\BP(U)}\sqrt{\Var_\BP(V)}\big| \label{varcov4}\\
	& = \big|\Cov_\BP(U, V)-\Var_\BP(U) +\sqrt{\Var_\BP(U)}\big(\sqrt{\Var_\BP(U)} -\sqrt{\Var_\BP(V)}\big)\big| \nonumber\\ 
	&\leq \big|\Cov_\BP(U, V)-\Var_\BP(U)\big| + \sqrt{\Var_\BP(U)} \times \big|\sqrt{\Var_\BP(U)} -\sqrt{\Var_\BP(V)}\big| \nonumber \\
	&\leq 2\sqrt{\Var_\BP(U)} \times \big|\sqrt{\Var_\BP(U)} -\sqrt{\Var_\BP(V)}\big|\label{cauchy}\\
	& \leq 2\sqrt{\Var_\BP(U)}\sqrt{\Expect_\BP{[(U-V)^2]}},\nonumber
	\end{align}
	where the penultimate line \eqref{cauchy} 
	 follows from the Cauchy-Schwarz inequality.
	Substituting \eqref{varcov4} into \eqref{varcov3},
	we get
	$$
	\Cov_\BP(U, V) \geq \sqrt{\Var_\BP(U)}\big(\sqrt{\Var_\BP(V)} - 2\sqrt{\Expect_\BP{[(U-V)^2]}}\big),
	$$
	which proves the claim. 
\end{proof}

The following sample variance concentration inequality will also come in handy.

\begin{lemma}[Equation 5, \citep{Maurer2009EmpiricalBB}]
	\label{variance}
	Let $ U $ be a random variable bounded by $ B $. Then for all $ \gamma > 0 $,
	$$
	\mathbb{P}\Big( \frac{n}{n-1}\widehat{\Var}(U)  \geq \Var(U)  - \gamma\Big) \geq 1-\exp\Big(-\frac{(n-1)\gamma^2}{8B^2\Var(U)}\Big).
	$$
	%
\end{lemma}

As explained in the main text, the key step in the proof of Proposition \ref{proposition lipschitz} is to apply Lemma \ref{corr} with a good approximation $ \widetilde{f_j}(\cdot) $ to the marginal projection $f_j(\cdot)$ that also has a sufficiently large number of data points in every one of its stationary intervals. The following lemma provides the precise properties of such an approximation.

\begin{lemma}\label{tilde-g}
	Suppose $ f_j(\cdot) $ satisfies Assumption \ref{assumption lipschitz}.
	Let $a>0$ be a positive constant and let $ M $ be a positive integer such that $ Ma \leq 1 $.
	There exists a function $\tilde f_j(\cdot)$ with at most $M$ stationary intervals such that, with probability at least $1 - 2n \exp(-na/12)$, both of the following statements are simultaneously true:
	\begin{enumerate}
		\item The number of data points in any stationary interval of $\tilde f_j(\cdot)$ is between $na/2$ and $ 3na/2 $.
		\item $\sqrt{\BE_n[(f_j(X_j)-\tilde f_j(X_j))^2]} \le K_0M^{-d} + K_1 (Ma)^{5/2}$,
		where $K_0$ and $K_1$ are some constants depending on $d$, $B$, $\alpha$, and $r$, and $\BE_n$ denotes the empirical expectation.
	\end{enumerate}
\end{lemma}

\begin{proof}[Proof of Lemma \ref{tilde-g}]
	Recall that $f_j(\cdot)$ is defined over $\BR$, so even though we restrict our attention $ [0, 1] $, we can assume it is bounded over the larger interval $[-1,2]$ for all $j$. 
	By \citep[Theorem VIII]{jackson1930approximation}, there exists a polynomial $ P_M(\cdot) $ of degree $ M+1>r$ 
	such that
	\begin{align*}
	\sup_{x \in [-1, 2]}|f_j(x)-P_M(x)|& \leq 3^d LA (M+1)^{-r}(M+1-r)^{-\alpha},
	\end{align*}
	where $ A = \frac{(r+1)^{r-1}(K/2)^r(K/2+2)}{r!} $,
	$L$ is the Lipschitz constant from 
	Assumption \ref{assumption lipschitz}, and $K $ is a universal constant given in \cite[Theorem I]{jackson1930approximation}. Since $ \frac{M+1}{M+1-r} \leq r+1 $ whenever $ M+1 > r$, we also have that
	\begin{equation} \label{triangle1}
	\sup_{x \in [-1, 2]}|f_j(x)-P_M(x)|\le   3^d(r+1)^{\alpha}LA (M+1)^{-d} \le K_0 M^{-d},
	\end{equation}
	where  $K_0 = 3^d(r+1)^\alpha LA$ is a constant. 

	
	Let $ \{ \xi_k\}_{1\leq k \leq M} $ be the collection of (at most) $ M $ stationary points of $ P_M(\cdot) $ in $[0,1] $. 
	We assume that $ a < \min_{k} (\xi_{k}-\xi_{k-1}) $;
	otherwise we remove points from $ \{\xi_k\}_{1\leq k \leq M} $ until this holds. Let $ I_k = [\xi_k, \xi_k+a] $.  
	
	We will now show that Properties 1 and 2 of Lemma \ref{tilde-g} are satisfied by the function
	\begin{equation} \label{tilde g def}
	\tilde f_j(x) = 
	\begin{cases}
	P_M(x) & x \notin \bigcup_k I_k\\
	P_M(\xi_k) & x \in I_k \; \text{for some}\; k.
	\end{cases}
	\end{equation}
	
	
	First, it is clear that $ \tilde f_j(\cdot) $ has at most $ M $ stationary intervals. Next, notice that by the multiplicative version of Chernoff's inequality \citep[Proposition 2.4]{angluin1979probabilistic}, since each stationary interval $ I_k $ has length $ a $, we have
	$$
	\mathbb{P}( \#\{ X_{ij} \in I_k\} \geq na/2) > 1-\exp(-na/8),
	$$
	and
	$$
	\mathbb{P}( \#\{ X_{ij} \in I_k\} \leq 3na/2) > 1-\exp(-na/12).
	$$
	Thus, by a union bound, the probability that 
	$$na/2 \leq   \#\{ X_{ij} \in I_k\} \leq 3na/2 \qquad \forall \; k$$ 
	is at least $ 1-2M\exp(-na/12)$. Since all stationary intervals are disjoint and are contained in $[0,1]$, we have $ Ma \leq 1 $ or $ M \leq 1/a $, which implies that $1-2M\exp(-na/12) \geq 1-(2/a)\exp(-na/12)$.
	Therefore, we know that 
	\begin{equation} \label{eq:between_prob}
	\mathbb{P}(3na/2 \geq  \#\{ X_{ij} \in I_k\} \geq na/2 \; \text{for all}\; k) \ge \max\{0, 1-(2/a)\exp(-na/12)\}.
	\end{equation}
	Notice that if $ a\leq 1/n $, then 
	$$\max\{0, 1-(2/a)\exp(-na/12)\} = 0 \geq 1 -2n \exp(-na/12)$$
	and if $ a > 1/n $, then 
	\begin{align*}
	\max\{0, 1-(2/a)\exp(-na/12)\} & \geq 1-(2/a)\exp(-na/12)  \\
	&\geq 1-2n\exp(-na/12).
	\end{align*}
	Thus, for all $ a \geq 0 $, we have 
	$$\max\{0, 1-(2/a)\exp(-na/12)\} \geq 1-2n\exp(-na/12),$$
	which when combined with \eqref{eq:between_prob}
	proves Property 1 of Lemma \ref{tilde-g}.
	
	To prove Property 2 of Lemma \ref{tilde-g} we use the triangle inequality and part of Property 1:
	\begin{equation}\label{chernoff}
	\mathbb{P}(\#\{ X_{ij} \in I_k\} \leq 3na/2\; \text{for all}\; k ) \ge 1-2n\exp(-na/12).
	\end{equation}
	In view of \eqref{triangle1}, to bound $\sqrt{\BE_n [(f_j(X_j)-\tilde f_j(X_j))^2]}$
	we aim to bound
	$
	\frac{1}{n}\sum_{i=1}^n (P_M(X_{ij}) - \tilde f_j(X_{ij}))^2.
	$
	
	Notice that  by Bernstein's theorem for polynomials \citep[Theorem B2a]{jackson1935bernstein}, we also have 
	\begin{equation} \label{bernstein}
	|P''_M(x)| \leq \frac{(M+1)^2\sup_{x\in [-1, 2]}|P_M(x)|}{(2-x)(x+1)}, \quad -1 < x < 2.
	\end{equation}
	By \eqref{triangle1},
	$ \sup_{x \in [-1, 2]}|P_M(x)| \leq K_0M^{-d}+ \sup_{x \in [-1, 2]}|f_j(X_j)| \leq  K_0 + B  $.
	Thus, by \eqref{bernstein},
$$
	\sup_{x\in [0,1]}|P''_M(x)| \leq \frac{(M+1)^2(K_0 + B)}{2} \le K_1M^2,
$$
	where $K_1 = 2(K_0 + B)$
	and we used the fact that $(2-x)(x+1)\geq 2$ for $x \in [0,1]$.  Additionally, because each $ \xi_k $ is a stationary point, $ P'_M(\xi_k) = 0 $, and hence by a second order Taylor approximation, we have 
	\begin{equation} \label{taylor}
	|P_M(x) - P_M(\xi_k)| \leq K_1M^2a^2/2
	\end{equation}
	for $ x \in I_k $.
	
	Therefore, by \eqref{tilde g def}, \eqref{chernoff}, and \eqref{taylor} we have that
	\begin{align*}
	\frac{1}{n}\sum_{i=1}^n (P_M(X_{ij}) - \tilde f_j(X_{ij}))^2 &= \sum_k \frac{1}{n}\sum_{X_{ij} \in I_k}(P_M(X_{ij}) - P_M(\xi_k))^2\\
	&\le \frac{K_1^2M^4a^4 }{4n} \sum_{k=1}^M \#\{X_{ij} \in I_k \}\\
	&\le \frac{K_1^2M^4a^4 }{4n} \sum_{k=1}^M \frac{3na}{2}\\
	&= \frac{3K_1^2(Ma)^5}{8},
	\end{align*}
	with probability at least $1-2n\exp(-na/12)$. Combining this with \eqref{triangle1} and
	using the triangle inequality to bound $ \sqrt{\BE_n [(f_j(X_j)-\tilde f_j(X_j))^2]} $ by $ \sqrt{\BE_n [(P_M(X_j)-\tilde f_j(X_j))^2]} + \sqrt{\vphantom{\BE_n [(P_M(X_j)-\tilde f_j(X_j))^2]}\BE_n [(P_M(X_j)-f_j(X_j))^2] }$ proves Property 2. 
	%
\end{proof}

Returning to the proof of Proposition \ref{proposition lipschitz}, by Lemma \ref{corr} along with Lemma \ref{tilde-g}, we have that with probability at least $ 1-2n\exp(-na/12)$ that
\begin{equation}\label{hoeffding}
\widehat\Delta(X_j, Y) \geq \frac{1}{2M/a + \log(2n)+1} \times\widehat{\Cov}^2\Bigg(\frac{\tilde f_j(X_j)}{\sqrt{\widehat{\Var}(\tilde f_j(X_j))}}, Y\Bigg).
\end{equation}

Now we choose
\begin{equation} \label{a,b,M}
a = \left(\frac{\tau^2\Var(f_j(X_j))}{4 (2K_0+K_1)^2}\right)^{(2d+5)/(10d)}, \quad M = \lfloor a^{-5/(2d+5)} \rfloor,
\end{equation}
where $\tau$ is a constant to be specified in \eqref{delta}, so that Property 2 of Lemma \ref{tilde-g} becomes 
\begin{equation}\label{g-g}
\begin{aligned} 
\sqrt{\BE_n [(f_j(X)-\tilde f_j(X))^2]} &\le K_0M^{-d} +K_1(Ma)^{5/2} \\
&\le (2K_0 +K_1) a^{5d/(2d+5)} \\
&\le   \frac{\tau \sqrt{\Var(f_j(X_j))}}{2},
\end{aligned}
\end{equation}
with probability at least $1 - 2n \exp(-na/12)$.
Recall the condition $M+1 > r$ as part of \citep[Theorem VIII]{jackson1930approximation}, which states that the degree of the polynomial must be greater than the order of Lipschitz derivative in order to approximate the function well. Since, $\Var(f_j(X_j)) \le B^2$,
this condition will be satisfied if we make the following choice:
\begin{equation} \label{delta}
\tau \coloneqq \min\Big(\frac{2(2K_0+K_1)}{r^{d} B} ,\frac{1}{4}\Big).
\end{equation}

It follows by Lemma \ref{measure} along with \eqref{g-g} that 
\begin{equation}\label{measure2}
\widehat{\Cov}(\tilde f_j(X_j), f_j(X_j)) \geq \sqrt{\widehat{\Var}(\tilde f_j(X_j))}\Big(\sqrt{\widehat{\Var}(f_j(X_j))} - \tau\sqrt{\Var(f_j(X_j))}\Big)
\end{equation}
with probability at least $1 - 2n \exp(-na/12)$.

In the next Lemma, we use \eqref{measure2} along with Lemma \ref{variance}
to obtain a lower bound on the right hand side of \eqref{hoeffding}. 

\begin{lemma}\label{long-lemma}
	With probability at least $1-\exp\Big(-\frac{(n-1)(1-8\tau^2)^2\Var(f_j(X_j))}{8B^2}\Big)-\exp\Big(-\frac{n\tau^2\Var(f_j(X_j))}{8(B^2+\sigma^2)} \Big) - 2 n \exp(-na/12)$, we have that
	\begin{align*}
	\widehat{\Cov}\Bigg(\frac{\tilde f_j(X_j)}{\sqrt{\widehat{\Var}(\tilde f_j(X_j))}}, Y\Bigg) \geq  (\tau/2)\sqrt{\Var(f_j(X_j))}.
	\end{align*}
\end{lemma}
\begin{proof}[Proof of Lemma \ref{long-lemma}]
	Recalling \eqref{cov}, we will prove Lemma \ref{long-lemma} by first getting a concentration bound on 
	(I)
	and then getting a concentration bound on 
	(II).
	
	To get a concentration bound on (I),
	we need Lemma \ref{variance} to lower bound the sample variance on the right hand side of inequality \eqref{measure2}. Choosing $U = f_j(X_j) \in [-B,B]$ and $ \gamma = \Var(f_j(X_j))(1-8\tau^2)$ (which is greater than zero by the choice of $ \tau $ in \eqref{delta}), notice that Lemma \ref{variance} gives us
	\begin{align*}
	&\mathbb{P}\Big(\widehat{\Var}(f_j(X_j)) \geq \frac{8\tau^2(n-1)}{n} \Var(f_j(X_j))\Big)  \\
	&\ge 1-\exp\Big(-\frac{(n-1)(1-8\tau^2)^2\Var(f_j(X_j))}{8B^2}\Big),
	\end{align*}
	so that by \eqref{measure2},
	\begin{align*}
	\widehat{\Cov}\Bigg(\frac{\tilde f_j(X_j)}{\sqrt{\widehat{\Var}(\tilde f_j(X_j))}}, f_j(X_j)\Bigg) & \geq \sqrt{\widehat{\Var}(f_j(X_j))} - \tau\sqrt{\Var(f_j(X_j))} \\
	& \geq \tau(\sqrt{8}\sqrt{1-1/n}-1)\sqrt{\Var(f_j(X_j))} \\
	& \geq \tau\sqrt{\Var(f_j(X_j))},
	\end{align*}
	with probability at least $1-\exp\Big(-\frac{(n-1)(1-8\tau^2)^2\Var(f_j(X_j))}{8B^2}\Big)-2n \exp(-na/12)$.
	
	Now we need to get a concentration bound for 
	(II). Let $ s_{i} = \frac{\tilde f_j(X_{ij})-\frac{1}{n}\sum_{k=1}^n\tilde f_j(X_{kj})}{\sqrt{\widehat{\Var}(\tilde f_j(X_j))}} $. We need to bound 
$$
	\widehat{\Cov}\Bigg(\frac{\tilde f_j(X_j)}{\sqrt{\widehat{\Var}(\tilde f_j(X_j))}}, Y-f_j(X_j)\Bigg) = \frac{1}{n}\sum_{i=1}^n s_{i}(g(\bX_i) - f_j(X_{ij})+ \varepsilon_i).
$$
	For notational simplicity, we let
	$
	\underline{\bX} = (X_{ij})
	$
	be the $n \times p $ data matrix with $\bX_i$ as rows.
	First notice that 
	\begin{align*}
	&\BE\Bigg[\exp\Bigg( \frac{\lambda}{n}\sum_{i=1}^n s_{i}( g(\bX_i) - f_j(X_{ij}) + \varepsilon_i)  \Bigg)\Bigg]\\
	& = \BE\Bigg[\BE\Bigg[\exp\Bigg( \frac{\lambda}{n}\sum_{i=1}^n s_{i}( g(\bX_i) - f_j(X_{ij})  + \varepsilon_i)   \Bigg) \Bigg\vert \underline{\bX} \Bigg] \Bigg]\\
	& = \BE\Bigg[\exp\Bigg( \frac{\lambda}{n}\sum_{i=1}^n s_{i} (g(\bX_i) - f_j(X_{ij})) \Bigg)\BE\Bigg[ \exp\left(\frac{\lambda}{n}\sum_{i=1}^n s_{i} \varepsilon_i \right) \Bigg\vert \underline{\bX} \Bigg] \Bigg].
	\end{align*}
	Now, by the sample independence of the errors $\epsilon_i$,
	we can write the above as
	\begin{align*}
	&\BE\Bigg[\exp\Bigg( \frac{\lambda}{n}\sum_{i=1}^n s_{i} (g(\bX_i) - f_j(X_{ij}))  \Bigg)\prod_{i=1}^n \BE\Big[ \exp\Big(\frac{\lambda}{n} s_{i} \varepsilon_i \Big) \Big\vert \underline{\bX} \Big] \Bigg] \\
	& \le \BE\Bigg[\exp\Bigg( \frac{\lambda}{n}\sum_{i=1}^n s_{i} (g(\bX_i) - f_j(X_{ij}))  \Bigg) \prod_{i=1}^n \exp\Big(\frac{\lambda^2 s_i^2 \sigma^2}{2n^2}\Big)\Bigg] \\
	& = \BE\Bigg[\exp\Bigg( \frac{\lambda}{n}\sum_{i=1}^n s_{i} (g(\bX_i) - f_j(X_{ij}))   \Bigg)  \Bigg]\exp\Big(\frac{\lambda^2 \sigma^2}{2n}\Big),
	\end{align*}
	where we used Assumption \ref{assumption epsilon} and the fact that $\frac{1}{n}\sum_{i=1}^n s^2_{i} = 1 $.
	Recalling that
	$s_i$ depends on $ (X_{1j}, X_{2j}, \dots, X_{nj})^{\top} $, we have
	\begin{equation} \label{old}
	\begin{aligned}
	&\BE\Bigg[\exp\Bigg( \frac{\lambda}{n}\sum_{i=1}^n s_{i} (g(\bX_i) - f_j(X_{ij}))   \Bigg)  \Bigg] \\
	& = \BE\Bigg[\BE\Bigg[\exp\Bigg( \frac{\lambda}{n}\sum_{i=1}^n s_{i} (g(\bX_i) - f_j(X_{ij}))   \Bigg) \Bigg\vert X_{1j}, X_{2j}, \dots, X_{nj}  \Bigg]\Bigg] \\
	& = \BE\Bigg[ \prod_{i=1}^n \BE\Big[\exp\Big( \frac{\lambda}{n} s_{i} (g(\bX_i) - f_j(X_{ij})) \Big) \Big\vert X_{1j}, X_{2j}, \dots, X_{nj}  \Big]\Bigg],
	\end{aligned}
	\end{equation}
	where we used sample independence in the second equality.
	Finally, applying Hoeffding's Lemma along with the fact that $\|g\|_\infty \le B$, we have that \eqref{old} is bounded above by
	\begin{align*}
	\BE\Bigg[\prod_{i=1}^n  \exp\bigg(\frac{\lambda^2 s_i^2B^2}{2n^2}\bigg) \Bigg] 
	\le \exp\bigg(\frac{\lambda^2 B^2}{2n} \bigg).
	\end{align*}
	Having bounded the moment generating function, we can now apply Markov's inequality to see that
	\begin{align*}
	&\BP \Bigg( \frac{1}{n}\sum_{i=1}^n s_{i}( g(X) - f_j(X_j) + \varepsilon_i) \le -\gamma\Bigg)\\
	&= \BP \Bigg( \exp\Bigg(-\frac{\lambda}{n}\sum_{i=1}^n s_{i}( g(X) - f_j(X_j) + \varepsilon_i) \Bigg) \ge \exp(\lambda \gamma)\Bigg) \\
	&\le \exp\Bigg(\frac{\lambda^2 (B^2+\sigma^2)}{2n} - \gamma \lambda\Bigg) \\
	& \le \exp\Bigg(-\frac{n\gamma^2}{2(B^2+\sigma^2)}\Bigg),
	\end{align*}
	where the last inequality follows by maximizing over $\lambda$. 
	Choosing $ \gamma = (\tau/2)\sqrt{\Var(f_j(X_j))} $, we have by a union bound that,
	\begin{align*}
	\widehat{\Cov}\Bigg(\frac{\tilde f_j(X_j)}{\sqrt{\widehat{\Var}(\tilde f_j(X_j))}}, Y\Bigg)
	& = \widehat{\Cov}\Bigg(\frac{\tilde f_j(X_j)}{\sqrt{\widehat{\Var}(\tilde f_j(X_j))}}, f_j(X_j)\Bigg) \\
	& \qquad+ \widehat{\Cov}\Bigg(\frac{\tilde f_j(X_j)}{\sqrt{\widehat{\Var}(\tilde f_j(X_j))}}, Y-f_j(X_j)\Bigg) \\
	& \geq \tau\sqrt{\Var(f_j(X_j))} - (\tau/2)\sqrt{\Var(f_j(X_j))} \\
	& = (\tau/2)\sqrt{\Var(f_j(X_j))},
	\end{align*}
	with probability at least $1-\exp\Big(-\frac{(n-1)(1-8\tau^2)^2\Var(f_j(X_j))}{8B^2}\Big)-\exp\Big(-\frac{n\tau^2\Var(f_j(X_j))}{8(B^2+\sigma^2)} \Big) - 2n \exp(-na/12)$. 
\end{proof}

With this setup, we are now ready to finish the proofs of Proposition \ref{proposition lipschitz} and \ref{proposition monotone}.

\begin{proof}[Proof of Proposition \ref{proposition lipschitz}]

	Recalling our concentration bound \eqref{hoeffding} 
along with Lemma \ref{long-lemma}, it follows that with probability at least
$1-\exp\Big(-\frac{(n-1)(1-8\tau^2)^2\Var(f_j(X_j))}{8B^2}\Big)-\exp\Big(-\frac{n\tau^2\Var(f_j(X_j))}{8(B^2+\sigma^2)} \Big)- 4n\exp\Big(-\frac{n}{12} \Big(\frac{\tau^2\Var(f_j(X_j))}{4 (2K_0+K_1)^2}\Big)^{(2d+5)/(10d)}  \Big)$ 
that
\begin{align}
\widehat\Delta(X_j, Y) & \geq \frac{1}{2M/a + \log(2n)+1} \times \widehat{\Cov}^2\Bigg(\frac{\tilde f_j(X_j)}{\sqrt{\widehat{\Var}(\tilde f_j(X_j))}}, Y\Bigg) \nonumber\\
&\geq \frac{\tau^2\Var(f_j(X_j))}{8a^{-(2d+10)/(2d+5)}  + 4(\log(2n)+1)} \nonumber\\
& \geq \frac{\tau^2\Var(f_j(X_j)) a^{(2d+10)/(2d+5)}}{8 + 4(\log(2n)+1) a^{(2d+10)/(2d+5)}} \nonumber\\
&\geq \frac{C_2 (\Var(f_j(X_j)))^{(6d+5)/5d} }{\log(n)}, \label{last}
\end{align}
where we used our choice of $M$ and $a$ in \eqref{a,b,M} and $\tau$ in \eqref{delta} and $C_2$ is some constant which only depends on $ L $, $ B $, $ r $, and $ \alpha$. Notice in the last inequality \eqref{last} we bound $a$ in the denominator with $\Var(f_j(X_j)) \le B^2$.

To conclude the proof, we simplify our probability bound. To this end, notice that
\begin{align*}
(\Var(f_j(X_j)))^{(2d+5)/(10d)} = \frac{\Var(f_j(X_j))}{(\Var(f_j(X_j)))^{(8d-5)/(10d)} } \ge  \frac{\Var(f_j(X_j))}{B^{(8d-5)/(5d)} },
\end{align*}
where the last inequality holds by assumption that $d \ge 5/8$ and the fact that $\Var(f_j(X_j)) \le B^2$. Therefore we have 
$
1-\exp\Big(-\frac{(n-1)(1-8\tau^2)^2\Var(f_j(X_j))}{8B^2}\Big)-\exp\Big(-\frac{n\tau^2\Var(f_j(X_j))}{8(B^2+\sigma^2)} \Big)- 4n \exp\Big(-\frac{n}{12} \Big(\frac{\tau^2\Var(f_j(X_j))}{4 (K_0+K_1)^2}\Big)^{(2d+5)/(10d)}  \Big)
\ge 1 - (4n+2) \exp(-nC_1\Var(f_j(X_j)))
$
for some constant $C_1$ which depends only on $L$, $B$, $\sigma$, $r$, and $\alpha$.
\end{proof}

\begin{proof}[Proof of Proposition \ref{proposition monotone}]


The main difference with Proposition \ref{proposition lipschitz} is that when $ f_j(\cdot) $ is monotone we now have $ M = 0 $. We no longer need the approximations $P_M(\cdot)$ or $\tilde f_j(\cdot)$ in Lemma \ref{tilde-g},  $a$ in \eqref{a,b,M} and $\tau$ in \eqref{delta}, or Lemma \ref{measure}. We will only need Lemma \ref{corr} and a version of Lemma \ref{long-lemma}.
Instead of applying Lemma \ref{corr} with an approximation to the marginal projection, we can choose $h_j(\cdot) $ to equal $ f_j(\cdot)$ directly to see that
\begin{equation} \label{corr-g}
\widehat\Delta(X_j, Y) \geq \frac{ \widehat{\Cov}^2\bigg(\frac{f_j(X_j)}{\sqrt{\widehat{\Var}(f_j(X_j))}},Y\bigg)}{\log(2n)+1} .
\end{equation}


Next, we have
\begin{equation} \label{cov-g}
\begin{aligned}
&\widehat{\Cov}\Bigg(\frac{f_j(X_j)}{\sqrt{\widehat{\Var}(f_j(X_j))}},Y\Bigg) \\
&= \widehat{\Cov}\Bigg(\frac{f_j(X_j)}{\sqrt{\widehat{\Var}(f_j(X_j))}},f_j(X_j)\Bigg) + \widehat{\Cov}\Bigg(\frac{f_j(X_j)}{\sqrt{\widehat{\Var}(f_j(X_j))}},Y-f_j(X_j)\Bigg) \\
&= \underbrace{\vphantom{\widehat{\Cov}\Bigg(\frac{f_j(X_j)}{\sqrt{\widehat{\Var}(f_j(X_j))}},Y-f_j(X_j)\Bigg)}\sqrt{\widehat{\Var}(f_j(X_j))}}_{\text{(IV)}} + \underbrace{\widehat{\Cov}\Bigg(\frac{f_j(X_j)}{\sqrt{\widehat{\Var}(f_j(X_j))}},Y-f_j(X_j)\Bigg)}_{\text{(V)}}.
\end{aligned}
\end{equation}
Now, we follow the same steps as the proof of Lemma \ref{long-lemma} to lower bound \eqref{cov-g}. For (IV), again use Lemma \ref{variance} with $U = f_j(X_j) \in [-B,B]$ but choose instead $ \gamma = \Var(f_j(X_j))/2$ to get
\begin{equation}\label{g}
\begin{aligned}
\mathbb{P}\Big(\text{(IV)} \geq \frac{\sqrt{\Var(f_j(X_j))}}{2} \Big) &\ge
\mathbb{P}\Big(\widehat{\Var}(f_j(X_j)) \geq \frac{n-1}{2n} \Var(f_j(X_j)) \Big)\\
&\geq  1-\exp\Big(-\frac{(n-1)\Var(f_j(X_j))}{32B^2}\Big).
\end{aligned}
\end{equation}
For (V), we can follow the same steps as the second half of the proof of Lemma \ref{long-lemma} to see that
$$
\BP (  \text{(V)} \le -\gamma)
\le \exp\Big(-\frac{n\gamma^2}{2(B^2+\sigma^2)}\Big).
$$
However, for \text{(V)}, we instead choose $\gamma = \sqrt{\Var(f_j(X_j))}/4$ to get
\begin{equation} \label{y-g}
\begin{aligned}
\BP \Bigg(  \text{(V)} \ge -\frac{\sqrt{\Var(f_j(X_j))}}{4}\Bigg)
\ge 1 - \exp\Bigg(-\frac{n\Var(f_j(X_j))}{32(B^2+\sigma^2)}\Bigg).
\end{aligned}
\end{equation}
Now using a union bound and substituting the events in \eqref{y-g} and \eqref{g} into \eqref{cov-g}, we
see that with probability at least  $ 1 - \exp\Big(-\frac{(n-1)\Var(f_j(X_j))}{32B^2}\Big) -  \exp\Big(-\frac{n\Var(f_j(X_j))}{32(B^2+\sigma^2)}\Big) \geq 1-2\exp\Big(-\frac{(n-1)\Var(f_j(X_j))}{32(B^2+\sigma^2)} \Big)$, we have that 
\begin{equation}\label{concentration-g}
 \widehat{\Cov}\Bigg(\frac{f_j(X_j)}{\sqrt{\widehat{\Var}(f_j(X_j))}},Y\Bigg)  \ge \frac{\sqrt{\Var(f_j(X_j))}}{4}.
\end{equation}

%
%

Therefore, substituting \eqref{concentration-g} into \eqref{corr-g}, we have that with probability at least
$1-2\exp\Big(-\frac{(n-1)\Var(f_j(X_j))}{32(B^2+\sigma^2)} \Big)$ that
$$
\widehat\Delta(X_j, Y) \geq \frac{ \widehat{\Cov}^2\bigg(\frac{\tilde f_j(X_j)}{\sqrt{\widehat{\Var}(\tilde f_j(X_j))}}, Y\bigg)}{\log(2n)+1}  \geq \frac{\Var(f_j(X_j))}{16 (\log(2n)+1)}. \qedhere
$$
\end{proof}

\section{Proof of Theorems \ref{theorem lipschitz ii} and \ref{theorem monotone ii}} \label{45}

In this section we use Proposition \ref{proposition lipschitz} and Proposition \ref{proposition monotone} along with Lemma \ref{delta-tail-2} to complete the proofs of Theorem \ref{theorem lipschitz ii} and Theorem \ref{theorem monotone ii}.


\begin{proof}[Proof of Theorem \ref{theorem lipschitz ii}]
The high-level idea is to show that the upper and lower bounds on the impurity reductions for irrelevant and relevant variables from Lemma \ref{delta-tail-2} and Proposition \ref{proposition lipschitz}, respectively, are well-separated.

	
	By Proposition \ref{proposition lipschitz} for all variables $j \in \calS$ and a union bound, we see that with probability at least  $1- s(4n+2)\exp(- C_1 n v)$,
	we have 
	\begin{equation}\label{delta-lower}
	\widehat\Delta(X_j, Y)  \ge  \frac{C_2v^{6/5+1/d} }{\log(n)} \qquad \forall \; j \in \calS.
	\end{equation}
	By Lemma \ref{delta-tail-2} and applying a union bound over all $ p-s $ variables in $ \calS^c $, we have that with probability at least $ 1- 4n(p-s)\exp(-n\xi^2/(12(B^2+\sigma^2))) $ that 
	\begin{equation}\label{delta-upper}
	\widehat\Delta(X_j, Y) \leq \xi^2 \qquad \forall \; j \in \calS^c .
	\end{equation}
	Recall that if we know the size $s$ of the support $\calS$, then $\mathcal{\widehat S}$ consists of the top $s$ impurity reductions.
	Note that choosing $\xi^2 = \frac{C_2v^{6/5+1/d} }{2\log(n)}$ in \eqref{delta-upper} will give us a high probability upper bound on $ \widehat\Delta(X_j, Y) $ for irrelevant variables which is dominated by the lower bound on $ \widehat\Delta(X_j, Y) $ for relevant variables in \eqref{delta-lower}. Thus, by a union bound, it follows that with probability at least $1-s(4n+2)\exp(- C_1 n v) - 4n(p-s)\exp\Big(-\frac{nC_2v^{6/5+1/d}}{24\log(n)(B^2+\sigma^2)}\Big)$, we have $\mathcal{\widehat S} = \mathcal{S}$.
\end{proof}

\begin{proof}[Proof of Theorem \ref{theorem monotone ii}]
The proof is similar to that of Theorem \ref{theorem lipschitz ii} except that we use Proposition \ref{proposition monotone} in place of Proposition \ref{proposition lipschitz}.

By Proposition \ref{proposition monotone} for all variables $j \in \calS$ and a union bound, we see that with probability at least  $1 - 2s\exp\Big(-\frac{(n-1)v}{32(B^2+\sigma^2)}\Big) $,
$$
\widehat\Delta(X_j, Y)  \ge  \frac{v}{4(1+\log(2n))} \qquad \forall \; j \in \calS.
$$
Again, we have by Lemma \ref{delta-tail-2} and a union bound over all $p-s$ variables in  $\calS^c$ that with probability at least $ 1- 4n(p-s)\exp(-n\xi^2/(12(B^2+\sigma^2))) $ that 
\begin{equation} \label{delta-upper-2}
\widehat\Delta(X_j, Y) \leq \xi^2 \qquad \forall \; j \in \calS^c. 
\end{equation}
Choosing $\xi^2 = \frac{ v}{8(1+\log(2n))}$ in \eqref{delta-upper-2} and using a union bound, it follows that with probability at least $1 - 2s\exp\Big(-\frac{(n-1)v}{32(B^2+\sigma^2)}\Big)- 4n(p-s)\exp\Big(-\frac{nv}{96(1+\log(2n))(B^2+\sigma^2)}\Big)$, we have $\mathcal{\widehat S} = \mathcal{S}$. 
\end{proof}

\bibliographystyle{plain}
\bibliography{selection.bib}

	

\end{document}